\documentclass[runningheads,envcountsame,orivec]{llncs}

\usepackage[T1]{fontenc}
\usepackage{etoolbox}
\usepackage{times}
\usepackage{soul}
\usepackage{url}
\usepackage[hidelinks]{hyperref}
\usepackage[utf8]{inputenc}
\usepackage[small]{caption}
\usepackage{graphicx}
\usepackage{amssymb}
\usepackage{amsmath}

\usepackage{amsthm}
\usepackage{booktabs}
\usepackage{algorithm}
\usepackage{algorithmic}
\usepackage{xparse}
\DeclareDocumentCommand{\todo}{o g}{\IfNoValueTF{#1}{\begingroup\color{magenta}TODO: #2\endgroup}{\begingroup\color{magenta}#1 #2\endgroup}}
\usepackage{graphicx}
\usepackage{xspace}
\usepackage{xparse}
\usepackage{enumerate}
\usepackage{tabularx}
\usepackage{booktabs}
\usepackage{csquotes}
\usepackage{latexsym}
\usepackage{stmaryrd}
\usepackage{wasysym}
\usepackage{multicol}
\usepackage{multirow}
\usepackage{subcaption}
\usepackage{arydshln}
\usepackage{rotating}

\usepackage{tikz}
\usetikzlibrary{fadings}
\usetikzlibrary{shapes,decorations,positioning}
\usetikzlibrary{fadings,shapes,decorations,positioning,calc}
\usetikzlibrary{decorations.markings,decorations.pathreplacing}
\usetikzlibrary{shapes.geometric,automata,positioning}
\usetikzlibrary{shadows}\usetikzlibrary{calc}
\usetikzlibrary{decorations,decorations.pathmorphing,decorations.pathreplacing}
\usetikzlibrary{calc}

\DeclareDocumentCommand{\todo}{o g}{\IfNoValueTF{#1}{\begingroup\color{magenta}TODO: #2\endgroup}{\begingroup\color{magenta}#1 #2\endgroup}}

\DeclareDocumentCommand{\remark}{o g}{\IfNoValueTF{#1}{\begingroup\color{magenta}Anmerkung: #2\endgroup}{\begingroup\color{magenta}#1 #2\endgroup}}

\theoremstyle{definition}

\newtheorem*{observation*}{Observation}
\makeatletter
\newcommand{\ifLatexThree}[2]{\@ifpackageloaded{xparse}{#1}{#2}}
\newcommand{\ifAMSmath}[2]{\@ifpackageloaded{amsmath}{#1}{#2}}
\newcommand{\ifMathSCR}[2]{\@ifpackageloaded{mathrsfs}{#1}{#2}}
\newcommand{\ifMathHyperREF}[2]{\@ifpackageloaded{hyperref}{#1}{#2}}
\makeatother

\ifLatexThree{%
	\NewDocumentCommand{\headword}{s o m}{\IfBooleanTF{#1}{#3}{\textbf{#3}}\IfNoValueTF{#2}{\index{#3}}{\index{#2}}}%
}{%
	\makeatletter
	\def\@headword#1{\textbf{#1}\index{#1}}%
	\def\@@headword#1{#1\index{#1}}%
	\def\headword#1{\@ifstar\@headword{#1}\@@headword{#1}}%
	\makeatother
}

\makeatletter
\@ifundefined{naturals}{}{}
\@ifundefined{ol}{}{}
\makeatother

\makeatletter
\newcommand{\textlabelmarker}[1]{%
	\protected@edef\@currentlabel{#1}%
	\phantomsection%
}
\newcommand{\textlabel}[2]{%
	\textlabelmarker{#1}%
	#1\label{#2}%
}
\makeatother

\ifx\nmableit\undefined
\makeatletter
\ifx\centernot\undefined
\newcommand*{\centernot}{%
	\mathpalette\@centernot
}
\fi
\def\@centernot#1#2{%
	\mathrel{%
		\rlap{%
			\settowidth\dimen@{$\m@th#1{#2}$}%
			\kern.5\dimen@
			\settowidth\dimen@{$\m@th#1=$}%
			\kern-.5\dimen@
			$\m@th#1\not$%
		}%
		{#2}%
	}%
}
\makeatother
\DeclareRobustCommand\nmableitSymb{\mathrel|\mkern-.5mu\joinrel\sim} %
\newcommand{\nmableit}{\ensuremath{\mbox{$\,\nmableitSymb\,$}}} %
\fi

\makeatletter
\ifMathHyperREF{%
	\newcommand{\seqref}[1]{\hyperref[{#1}]{\textup{\tagform@split{\getrefnumber{#1}}}}}%
}{%
	\newcommand{\seqref}[1]{\textup{\tagform@split{\getrefnumber{#1}}}}%
}
\newcommand\tagform@split[1]{%
	\begingroup
	\m@th\normalfont(\ignorespaces #1\unskip\@@italiccorr)%
	\endgroup
}
\makeatother

	\makeatletter
	\newcommand{\leqnomode}{\tagsleft@true\let\veqno\@@leqno}
	\newcommand{\reqnomode}{\tagsleft@false\let\veqno\@@eqno}
	\newcommand{\pushright}[1]{\ifmeasuring@#1\else\omit\hfill$\displaystyle#1$\fi\ignorespaces}
	\newcommand{\pushleft}[1]{\ifmeasuring@#1\else\omit$\displaystyle#1$\hfill\fi\ignorespaces}
	\newcommand{\specialcell}[1]{\ifmeasuring@#1\else\omit$\displaystyle#1$\ignorespaces\fi}

	\makeatother
\newcommand{\ksIF}{\text{if }}
\newcommand{\ksTHEN}{\text{, then }}

\newcommand{\ksAND}{\text{ and }}
\newcommand{\ksOR}{\text{ or }}

\newcommand{\ksOtherwise}{\text{otherwise}}

\newcommand{\tuple}[1]{\ensuremath{\langle{#1}\rangle}}

\newcommand{\modelsOf}[1]{\ensuremath{\llbracket #1\rrbracket}}
\newcommand{\negOf}[1]{{\ensuremath{\neg{#1}}}}

\newcommand{\minOf}[2]{\ensuremath{\min(#1,#2)}}

\newcommand{\nAtom}[1]{{\ensuremath{\overline{#1}}}}

\renewcommand{\modelsOf}[1]{\ensuremath{\ksMod(#1)}}

\DeclareMathOperator{\Cn}{Cn}
\DeclareMathOperator{\Th}{Th}

\DeclareMathOperator{\ksTh}{Th}
\newcommand{\theoryOf}[1]{\ensuremath{\ksTh(#1)}}
\DeclareMathOperator{\ksBel}{Bel}
\newcommand{\beliefsOf}[1]{\ensuremath{\ksBel(#1)}}

\newcommand{\setAllES}{\ensuremath{\mathcal{E}}}

\newcommand{\propLang}{\ensuremath{\mathcal{L}}}

\ifMathSCR{%
}{%
}

\makeatletter
\newcommand\footnoteref[1]{\protected@xdef\@thefnmark{\ref{#1}}\@footnotemark}
\makeatother

\usepackage{environ,xparse,xkeyval,ifthen}
\newif\ifpostulatepresent\postulatepresentfalse

\ExplSyntaxOn
\NewEnviron{postulate}[2]%
{%
	\tl_gclear:N \g_tmpa_tl%
	\tl_gput_right:Nn \g_tmpa_tl { \ifpostulatepresent\\[\smallskipamount]\fi\global\postulatepresenttrue }%
	\tl_gput_right:Nn \g_tmpa_tl { \noindent\ignorespaces\ifthenelse{\equal{#1}{}}{}{(\textlabel{#1}{pstl:#1})} & }%
	\tl_gput_right:No \g_tmpa_tl { \BODY }%
	\tl_gput_right:Nn \g_tmpa_tl { & \hspace{1.5\medskipamount} #2  }%
	\group_insert_after:N \g_tmpa_tl%
}
\ExplSyntaxOff

\DeclareDocumentCommand{\revision}{}{*}
\DeclareDocumentCommand{\clRevision}{}{\ensuremath{\circledast}}
\DeclareMathOperator{\supmin}{\supmin}

\DeclareMathOperator{\beliefSymb}{Bel}
\newcommand{\epistemicSpace}{\ensuremath{\mathbb{E}}}
\newcommand{\logicLanguage}{\ensuremath{\mathcal{L}}}

\renewcommand{\modelsOf}[1]{\ensuremath{\llbracket#1\rrbracket}}

\newcommand{\logicBeliefSet}{\logicLanguage^{\beliefSymb}}

	\makeatletter
\renewcommand{\textlabel}[2]{%
	\protected@edef\@currentlabel{#1}%
	\phantomsection%
	#1\label{#2}%
}
\makeatother

\begin{document}
\title{Credibility-Limited Revision for Epistemic Spaces\\
(including supplementary material)}
\titlerunning{Credibility-Limited Revision for Epistemic Spaces}
\author{
    Kai Sauerwald\orcidID{0000-0002-1551-7016}}
\authorrunning{Kai Sauerwald}
\institute{FernUniversität in Hagen,
    Artificial Intelligence Group, 58084~Hagen,~Germany\\
    \email{kai.sauerwald@fernuni-hagen.de}}
\maketitle              %

\begin{abstract}
    We consider credibility-limited revision in the framework of belief change for epistemic spaces, permitting inconsistent belief sets and inconsistent beliefs.
    In this unrestricted setting, the class of credibility-limited revision operators does not include any AGM revision operators.
    We extend the class of credibility-limited revision operators in a way that all {AGM} revision operators are included while keeping the original spirit of credibility-limited revision.
    Extended credibility-limited revision operators are defined axiomatically. A semantic characterization of extended credibility-limited revision operators that employ total preorders on possible worlds is presented.
\end{abstract}

\section{Introduction}
\label{sec:introduction}
Much research in belief change theory is on the change of logical theories \cite{KS_FermeHansson2018}. A well-known and widely accepted approach for the revision of logical theories is revision by Alchourron, Gärdenfors and Makinson \cite{KS_AlchourronGaerdenforsMakinson1985} (AGM), which realizes the famous principle of minimal change. 
Another belief change operation in this setting is credibility-limited revision by Hansson, Fermé, Cantwell and Falappa \cite{KS_HanssonFermeCantwellFalappa2001}. This class of operations implements the idea that an (AGM) revision should performed only when the newly arriving information is credible and if the information is not credible, the agent's beliefs are not altered.
Intuitively, credibility-limited revision is a generalization of AGM revision; when one considers all potential information as credible, one would expect that a credibility-limited revision \emph{is} an AGM revision. 

Apart from the classical setting of theory change, belief change is considered in the more general setting of belief change over epistemic states by Darwiche and Pearl \cite{KS_DarwichePearl1997,KS_SchwindKoniecznyPinoPerez2022,KS_Sauerwald2022}. 
In this setting, which has wide applications in iterated belief change \cite{KS_FermeHansson2018}, one does not only consider the beliefs of an agent but also considers extra logical information that guides the belief change process as part of the representation.
To deal with this expressive setting, both above-mentioned kinds of belief changes have been adapted to this setting, i.e., AGM revision by Darwiche and Pearl \cite{KS_DarwichePearl1997} and credibility limited revision by Booth, Ferm{\'{e}}, Konieczny and {Pino P{\'{e}}rez \cite{KS_BoothFermeKoniecznyPerez2012}.
A recent clarification of the Darwiche and Pearl framework is the framework of \emph{belief change for epistemic spaces}~\cite{KS_SchwindKoniecznyPinoPerez2022}.  
Agents' epistemic states are bound to a specific type of representation, and an epistemic space is an abstraction that describes the whole room of all possible epistemic states of an individual agent. Belief change operators for an epistemic space reside within these representational bounds.  
We consider what is called here the unrestricted framework of belief change for epistemic spaces, which means that inconsistent beliefs are permitted, these are often neglected but not always \cite{KS_FermeWassermann2018}.

This paper starts with the observation that when using the unrestricted framework of belief changes for epistemic spaces, the given notion of credibility-limited revision \emph{does not} behave very well in the unrestricted case; all AGM revision operators are excluded, and inconsistent belief sets cannot be handled.
We deal with this observation by providing the following results, which are also the main contributions%
\footnote{Some of these results are already part of the dissertation thesis by the author \cite{KS_Sauerwald2022}.}:
\begin{itemize}
    \item {[Extended Credibility-Limited Revision]} We define \emph{extended credibility-limited revision}, which builds upon credibility-limited revision by Booth et al. \cite{KS_BoothFermeKoniecznyPerez2012}.
    For this, we consider the axiomatic description of credibility-limited revision by Booth et al. and identify one postulate that makes these exclude AGM revision operators and incompatible with inconsistent beliefs. 
    For defining extended credibility-limited revision, we add two postulates to the original postulates by Booth et al. for credibility-limited revision (and remove the postulate which makes them incompatible with AGM revision).
    The additional postulates ensure that operators are excluded which do not match the intuition of credibility-limited revision. 
    
    \item {[Semantic Characterization]} A semantical characterization of extended credibility-limited revision. This characterization is given in terms of functions that assign total preorders to epistemic states, i.e., in 
    the same style as the Darwiche-Pearl representation theorem for revision \cite{KS_DarwichePearl1997}, respectively as in the semantic characterization of credibility-limited revision by Booth et al. \cite{KS_BoothFermeKoniecznyPerez2012}.
    
    \item {[Genuineness]} We show that extended credibility-limited revisions are a genuine extension of credibility-limited revisions by Booth et al. \cite{KS_BoothFermeKoniecznyPerez2012} that include all AGM revision operators.
\end{itemize}
The paper contains the proofs for all propositions and theorems given here. The next section gives the background on propositional logic and order theory. In Section~\ref{sec:epistemic_spaces}, we present epistemic spaces, as well as AGM revision operators for epistemic states \cite{KS_DarwichePearl1997} and credibility-limited revision operators by Booth et al. \cite{KS_BoothFermeKoniecznyPerez2012}.
We observe in Section~\ref{sec:problem} that credibility-limited revision for epistemic spaces does not include AGM revisions for epistemic spaces.
Section~\ref{sec:GCL} introduces extended credibility-limited revision and we consider a semantic characterization of this class of operators.
An example of extended credibility-limited revision is given in Section~\ref{sec:examples_properties}, and we consider some properties of extended credibility-limited revision operators. The last section, Section~\ref{sec:conclusion}, summarises the results presented here.

Before starting with the main content of the paper, we consider some remarks. 
This paper is mainly developed from a technical perspective, and after the introduction we do \emph{not} focus on discussing applications and implications of the results given here and delegate such a discussion to different paper.
From a theoretical perspective, we should be interested in considering belief changes on \emph{arbitrary} epistemic spaces and arbitrary inputs, as we do in this paper. 
A rationale is that this allows us to study belief change independent of specific representations of epistemic states, respectively, in a way that the results apply to \emph{all} possible representations, including those with inconsistent belief sets. 
Doing so has the advantage that the theory applies to application scenarios that have not been anticipated.
One application of this is employing belief change operators as descriptional theories, which is, in my opinion, a prerequisite for using belief change theory in, e.g., approaches like cognitive logics \cite{KS_RagniKernIsbernerBeierleSauerwald2020}.
In that sense, the purpose of this paper goes beyond just generalizing credibility-limited revision; it exemplifies how to generalise belief change operators to arbitrary inputs and representations.

\section{Background}
\label{sec:background}
Let \( \Sigma \) be a non-empty finite propositional signature whose elements are called atoms.
With \( \logicLanguage \) we denote the set of all propositional formulas over \( \Sigma \) defined as usually using Boolean connectives.
We assume that the tautology \( \top \) and the falsum \( \bot \) are elements of \( \logicLanguage \).
The set of all \( \Sigma \)-interpretations is denoted by \( \Omega \) and we write interpretations as strings of atoms from \( \Sigma \) where an bar over an atom indicates that this atom is mapped to \texttt{false} and otherwise to \texttt{true}.
For instance, the interpretation \( \omega = a\overline{b}c \) maps \( a \) to \texttt{true} and \( b \) to \texttt{false} and \( c \) to \texttt{true}.
The models relation \( \models \) between interpretations and formulas is defined as usually and with \( \modelsOf{\alpha}=\{ \omega\in\Omega \mid \omega\models \alpha  \} \)  we denote the set of all models of \( \alpha \).
We say a formula \( \alpha \) logically entails a formula \( \beta \), written \( \alpha \models \beta \), if \( \modelsOf{\alpha}\subseteq\modelsOf{\beta} \) holds.
These notions are lifted to sets of formulas \( X\subseteq \logicLanguage \) as usually, i.e., \( \modelsOf{X}=\bigcap_{\alpha\in X}\modelsOf{\alpha}\) and \(  X\models\alpha\) if \( \modelsOf{X}\subseteq\modelsOf{\alpha} \).
We say that \( X\subseteq \logicLanguage \) is deductively closed if \( X=\Cn(X) \), whereby \( \Cn(X)=\{ \alpha \in \logicLanguage \mid X \models \alpha \} \) is the closure under logical entailment. With \( \logicBeliefSet \) we denote the set of all deductively closed sets.
For \( \alpha\in\logicLanguage \) we define \(X+\alpha=\Cn(X\cup\{\alpha\})  \).
Moreover, for \( M\subseteq \Omega \) we define \( \theoryOf{M}=\{ \alpha\in\logicLanguage\mid M \subseteq \modelsOf{\alpha} \} \).
 A formula \( \alpha\in\logicLanguage \), respectively a set \( X\subseteq \logicLanguage \), is called consistent if \( \modelsOf{\alpha}\neq\emptyset \), respectively \( \modelsOf{X}\neq\emptyset \).
A total preorder \( {\preceq} \) on subset \( M \subseteq \Omega \) is a relation \( {\preceq}\subseteq M \times M \) such that \( {\preceq} \) is total, i.e., for all \( \omega_1,\omega_2 \in M \) holds \( \omega_1 \preceq \omega_2 \) or \( \omega_2\preceq\omega_1 \), and transitive, i.e.,  for all \( \omega_1,\omega_2,\omega_3 \in M \) holds that  \( \omega_1\preceq\omega_2 \) and \( \omega_2\preceq\omega_3 \) imply \( \omega_1\preceq\omega_3 \). Note that totality implies that \( {\preceq} \) is reflexive, i.e., \( \omega\preceq\omega \) holds for all \( \omega\in M \).
A total preorder \( {\ll} \) on \( M \subseteq \Omega \) is called a \emph{linear order}, if \( {\ll} \) is antisymmetric, i.e., for all \( \omega_1,\omega_2 \in M \) holds that \( \omega_1\ll\omega_2 \) and \( \omega_2\ll\omega_1 \) imply \( \omega_1 = \omega_2 \).
The set of minimal elements of \( X\subseteq \Omega \) with respect to \( {\preceq} \) is 
\( \min(X,{\preceq})=\{  \omega\in X \mid \omega \preceq \omega' \text{ for all } \omega'\in X \} \) 
and 
 \( {\simeq} \) denotes the equivalent part of \( {\preceq} \).

\section{Background on Belief Change for Epistemic Spaces}
\label{sec:epistemic_spaces}
\label{sec:es_subsec}
In this work, we model agents by the means of logic. 
Deductive closed sets of formulas,  which we denote from now as \emph{belief set}, represent deductive capabilities.
The interpretations represent worlds that the agent is capable to imagine.
The following notion describes the space of epistemic possibilities of an agent's mind in a general way.
\begin{definition}[\cite{KS_SchwindKoniecznyPinoPerez2022}; adapted]\label{def:epistemic_space}
    A tuple 
    \( \epistemicSpace
    =
    \langle
    \setAllES,\beliefSymb\rangle  \) is called an \emph{epistemic space} if $ \setAllES $ is a non-empty set 
    and $ \beliefSymb : \setAllES \to \logicBeliefSet $.
\end{definition}
\noindent We 
call the elements of \( \setAllES \) \emph{epistemic states} and \pagebreak[3]
use \( \modelsOf{\Psi} \) as shorthand for \( \modelsOf{\beliefsOf{\Psi}} \). %
Within this framework belief change operators are transitions from one epistemic state to another when new beliefs are received, i.e., 
belief change operators for an epistemic space \( \epistemicSpace \) are global objects, functions on all epistemic states in the mathematical sense.
\begin{definition}    
    A \emph{belief change operator for an epistemic space \( \epistemicSpace = \langle \setAllES,\beliefSymb \rangle  \)}  is a function $ \circ : \setAllES \times \logicLanguage \to  \setAllES $.
\end{definition}
The framework of belief change for epistemic spaces can be instantiated to often-considered settings of belief change.
When \( \setAllES \) is the set of all belief sets over \( \logicLanguage \) and \( \beliefsOf{\Psi}=\Psi \), one obtains the classical setting of theory change \cite{KS_AlchourronGaerdenforsMakinson1985}, respectively the setting considered by Katsuno and Mendelzon \cite{KS_KatsunoMendelzon1992}.
In iterated belief change, typical instantiations for \( \epistemicSpace \) are ranking functions by Spohn \cite{KS_Spohn1988} or total preorders~\cite{KS_DarwichePearl1997}.
The notion of an epistemic space by Schwind et al. \cite{KS_SchwindKoniecznyPinoPerez2022} slightly differs from the notion here insofar that here, we \emph{do permit} inconsistent beliefs (cf. Definition~\ref{def:epistemic_space}).
For that reason we denote the framework considered here as \emph{unrestricted}.
We can (nearly) obtain the restricted setting by considering only consistent formulas and demanding that an epistemic space \( \epistemicSpace = \langle \setAllES,\beliefSymb \rangle  \) satisfies the following condition:
\begin{align}
    & \text{If } \Psi\in\setAllES \ksTHEN \beliefsOf{\Psi} \neq \Cn(\bot)	\tag{global consistent}\label{pstl:globalconsistency}
\end{align}

Clearly, to study types of belief changes, one restricts the space of all belief change operators for an epistemic spaces to specific classes of operators.
In the following, we consider such classes of operators.

\label{sec:agm_subsec}
\medskip\noindent\textbf{AGM Revision.}
Revision is the process of incorporating new beliefs into an agent's belief set while maintaining consistency, whenever this is possible.
We use an adaptation of the AGM postulates for revision~\cite{KS_AlchourronGaerdenforsMakinson1985}
for the framework of epistemic spaces {\cite{KS_DarwichePearl1997}, which is inspired by the approach of Katsuno and Mendelzon \cite{KS_KatsunoMendelzon1992}.
        A belief change operator \( \revision \) for an epistemic space \( \epistemicSpace=\tuple{\setAllES,\beliefSymb} \) is called an \emph{(AGM) revision operator  for \( \epistemicSpace \)} if the following postulates are satisfied
        {\cite{KS_DarwichePearl1997}}: %
        \begin{description}
            \item[\normalfont(\textlabel{R1}{pstl:R1})] 
                \( \alpha \in  \beliefsOf{\Psi \revision \alpha}  \)
            \item[\normalfont(\textlabel{R2}{pstl:R2})] 
                 \(   \beliefsOf{\Psi \revision \alpha} = {\beliefsOf{\Psi} + \alpha } \text{ if } \beliefsOf{\Psi} + \alpha  \text{ is consistent } \)
            \item[\normalfont(\textlabel{R3}{pstl:R3})] 
                If \(  \alpha \text{ is consistent, then }  \beliefsOf{\Psi \revision \alpha} \text{ is consistent}  \)
            \item[\normalfont(\textlabel{R4}{pstl:R4})] 
                If \( \alpha \equiv \beta \ksTHEN  \beliefsOf{\Psi \revision \alpha} = \beliefsOf{\Psi \revision \beta}  \)
            \item[\normalfont(\textlabel{R5}{pstl:R5})] 
                \( \beliefsOf{\Psi \revision (\alpha\land\beta)} \subseteq \beliefsOf{\Psi \revision \alpha}+\beta \)
            \item[\normalfont(\textlabel{R6}{pstl:R6})] 
                If \(  \beliefsOf{\Psi\revision\alpha} + \beta \) is consistent, then \(  \beliefsOf{\Psi \revision \alpha}+\beta \subseteq \beliefsOf{\Psi \revision (\alpha\land\beta)}  \)
        \end{description}
    AGM revision is well-known for realizing the principle of minimal change on the prior beliefs when revising. 
    Note that AGM revision in the setting epistemic spaces is expressible, as the model is Turing complete \cite{KS_SauerwaldBeierle2022b}.
    However, in some epistemic spaces no AGM revision operator exist at all \cite{KS_SauerwaldThimm2024}.
\label{sec:clr_subsec}
\medskip\noindent\textbf{Credibility-Limited Revision.}
Credibility-limited revision was introduced by Hansson et al.~\cite{KS_HanssonFermeCantwellFalappa2001} and restricts the process of revision to credible beliefs.
To deal with epistemic states,   \emph{credibility-limited revision} was adapted by Booth et al. \cite{KS_BoothFermeKoniecznyPerez2012}.
A belief change operator \( \clRevision \) for an epistemic space \( \epistemicSpace=\tuple{\setAllES,\beliefSymb} \) is called an \emph{credibility-limited revision operator  for \( \epistemicSpace \)} if the following postulates are satisfied
 {\cite{KS_BoothFermeKoniecznyPerez2012}}:\pagebreak[3]
\begin{description}
	\item[\normalfont(\textlabel{CL1}{pstl:CLR1})] $ \alpha\in \beliefsOf{\Psi\clRevision\alpha} \ksOR \beliefsOf{\Psi\clRevision\alpha} = \beliefsOf{\Psi} $
	\item[\normalfont(\textlabel{CL2}{pstl:CLR2})]$ \beliefsOf{\Psi{\clRevision}\alpha} = \beliefsOf{\Psi}+\alpha $ if $ \beliefsOf{\Psi}+\alpha $ is consistent
	\item[\normalfont(\textlabel{CL3}{pstl:CLR3})] $ \beliefsOf{\Psi\clRevision\alpha} $ is consistent
	\item[\normalfont(\textlabel{CL4}{pstl:CLR4})] If $ \alpha\equiv\beta $, then $ \beliefsOf{\Psi\clRevision\alpha} = \beliefsOf{\Psi\clRevision\beta} $
	\item[\normalfont(\textlabel{CL5}{pstl:CLR5})] If $ \alpha \! \in\! \beliefsOf{\Psi\clRevision\alpha} $ and $ \alpha\models\beta $, then $ \beta \in \beliefsOf{\Psi\clRevision\beta} $
	\item[\normalfont(\textlabel{CL6}{pstl:CLR6})] $ \beliefsOf{\Psi\!\clRevision\!(\alpha\lor\beta)}\! =\!\begin{cases}
		\beliefsOf{\Psi\clRevision\alpha} \text{ or} \\
		\beliefsOf{\Psi\clRevision\beta} \text{ or} \\
		\beliefsOf{\Psi\clRevision\alpha} \cap \beliefsOf{\Psi\clRevision\beta}
	\end{cases}  $
\end{description}
The postulate \eqref{pstl:CLR1} is known as \emph{relative success} and denotes that either the agent keeps its prior beliefs (falling back to prior beliefs) or the belief change is successful in achieving the success condition of revision (the beliefs get accepted for revision).
Through \eqref{pstl:CLR2}, known as \emph{vacuity}, new beliefs are just added when they are not in conflict with $ \beliefsOf{\Psi} $.
The postulate \eqref{pstl:CLR3}, also known as \emph{strong consistency} \cite{KS_HanssonFermeCantwellFalappa2001}, ensures consistency, and by \eqref{pstl:CLR4} the operator has to implement independence of syntax.
Postulate \eqref{pstl:CLR5} guarantees that when the revision by a belief $ \alpha $ is successful, then it is also successful for every more general belief $ \beta $.
The trichotomy postulate \eqref{pstl:CLR6} guarantees decomposability of revision of disjunctive beliefs.

\section{Observations on AGM Revision and Credibility-Limited Revision in the Unrestricted Framework}
\label{sec:problem}

The approach for credibility-limited revision for epistemic spaces, as given by Booth et al. (cf. Section~\ref{sec:clr_subsec}), is made with the restriction to consider only consistent beliefs. In the unrestricted framework of epistemic spaces, we also permit inconsistent beliefs, and next, we observe now that in these cases, no credibility-limited revision exists at all.
\begin{proposition}\label{prop:no_cl_inunrestricted}
    Let \( \epistemicSpace=\tuple{\setAllES,\beliefSymb} \) be an 
    epistemic space and let  $ \clRevision $ be a belief change operator for \( \epistemicSpace \). 
    If \( \epistemicSpace \) is not  \hyperref[pstl:globalconsistency]{globally consistent}, then \( \clRevision \) is not a credibility-limited revision operator.
\end{proposition}
\begin{proof}
    If \( \epistemicSpace \) is not  \hyperref[pstl:globalconsistency]{globally consistent}, then there is some epistemic state \( \Psi_{\bot}\in\setAllES \)  with \( \modelsOf{\Psi_{\bot}}=\emptyset \), i.e., \( \beliefsOf{\Psi_{\bot}}=\Cn(\bot) \).
    Suppose now that \( \clRevision \) is a credibility-limited revision operator.
    Because of that \( \clRevision \) satisfies \eqref{pstl:CLR1} and \eqref{pstl:CLR3}.
    From \eqref{pstl:CLR3}, we obtain that \( \modelsOf{\Psi_{\bot}\clRevision\bot}\neq\emptyset \) holds.
    This is a contradiction, because due to \eqref{pstl:CLR1}, we also have that \(  \modelsOf{\Psi_{\bot}\clRevision\bot} =\emptyset \) holds.
\end{proof}

When consider belief changes in the unrestricted framework of epistemic spaces, we observe that AGM revision operators are not credibility-limited revision operators.
\begin{proposition}\label{prop:clr_violate_agm_rev}
    Let \( \epistemicSpace=\tuple{\setAllES,\beliefSymb} \) be an epistemic space.
    Every AGM revision operator for \( \epistemicSpace \) is not a credibility-limited revision operator for \( \epistemicSpace \).
\end{proposition}
\begin{proof}
For each AGM revision operator \( \revision \) for epistemic spaces holds \( \modelsOf{\Psi\revision\bot}=\emptyset \)  due to \eqref{pstl:R1} (as in the setting of theory change). 
Because of that, \( \revision \) violates \eqref{pstl:CLR3}, as \eqref{pstl:CLR3} demands that \( \modelsOf{\Psi\revision\bot}\neq\emptyset \) holds.
Consequently, \( \revision \) is not a credibility-limited revision operator.
\end{proof}

To describe Proposition~\ref{prop:clr_violate_agm_rev} from the viewpoint of classes of operators, we define the respective classes of operators.
With \( \textsf{AGMRev}(\epistemicSpace) \) we denote the class of all AGM revision operators for \( \epistemicSpace \), i.e., 
    \( \textsf{AGMRev}(\epistemicSpace) = \{\     \revision:\setAllES\times \logicLanguage \to \setAllES \mid \revision \text{ satisfies \eqref{pstl:R1}--\eqref{pstl:R6}}     \ \} \),
and with \( \textsf{CLRev}(\epistemicSpace) \) we denote the class of all credibility-limited revision operators for \( \epistemicSpace \), i.e., 
    \( \textsf{CLRev}(\epistemicSpace) = \{\   \clRevision  :\setAllES\times \logicLanguage \to \setAllES \mid \clRevision \text{ satisfies \eqref{pstl:CLR1}--\eqref{pstl:CLR6}}     \ \} \).
Proposition~\ref{prop:clr_violate_agm_rev} yields the following results.
\begin{corollary}
    For each epistemic space \( \epistemicSpace \) holds:
    \begin{equation*}
        \textsf{AGMRev}(\epistemicSpace) \cap \textsf{CLRev}(\epistemicSpace) = \emptyset
    \end{equation*}
\end{corollary}

\section{Extended Credibility-Limited Revision}
\label{sec:GCL}
In the following, we extend credibility-limited revision \cite{KS_BoothFermeKoniecznyPerez2012} so that AGM revision operators are not excluded in the unrestricted framework of epistemic spaces and that operators exist, even when inconsistent beliefs are permitted.
At first, we will observe that just dropping \eqref{pstl:CLR3} on the postulate side will include belief change operators with undesired behaviour. 
We introduce two postulates that exclude operators with undesired behaviour, which are meant to replace \eqref{pstl:CLR3}. By employing these postulates we define extended credibility-limited revision.
This sections ends with a semantic characterization of extended credibility-limited revision.

\subsection{Credibility-Limited Revision Without \eqref{pstl:CLR3}}
In Section~\ref{sec:problem}, we showed that AGM revision operators are not credibility-limited revision operators in the unrestricted stetting of belief change for epistemic spaces and when inconsistent beliefs are permitted, no credibility-limited revision operator exists. The cause for this is the postulate \eqref{pstl:CLR3}  of credibility-limited revision, e.g., AGM revision operators are incompatible with the postulate \eqref{pstl:CLR3}.
However, excluding \eqref{pstl:CLR3}, respectively by just taking \eqref{pstl:CLR1}, \eqref{pstl:CLR2}, and \eqref{pstl:CLR4}--\eqref{pstl:CLR6}, we would observe drastic consequences, because we would permit operators that would yield randomly inconsistent states for certain inputs. The following example contains a fairly simple operator which has such a behaviour.
\begin{example}\label{ex:cl:ncl3unwanted}
	Let \( \Sigma=\{a\} \) and let \( \epistemicSpace_{\bot,a}=\tuple{\setAllES,\beliefSymb} \) be the epistemic space given by:
	\begin{align*}
	\setAllES & =\{ \Psi_\bot,\Psi_{a} \} &	\modelsOf{\Psi_\bot} & = \emptyset & \modelsOf{\Psi_{a}}      & = \{ ab \}      \ .
	\end{align*}     
    Note that the function \( \beliefSymb \) is implicitly defined via \( \beliefsOf{\Psi}=\Th(\modelsOf{\Psi}) \).
	We define a belief change operator \( \clRevision \) for \( \epistemicSpace_{\bot,a} \) as follows:
	\begin{equation*}
		\Psi \clRevision \alpha = \begin{cases}
			\Psi_{a} &\ksIF \modelsOf{\alpha}=\{{a}\} \\
			\Psi_{\bot} &\ksIF \modelsOf{\alpha}=\{\nAtom{a}\} \ksOR \modelsOf{\alpha}=\emptyset \\
			\Psi & \ksIF \modelsOf{\alpha}=\{{a},\nAtom{a}\}
		\end{cases}
	\end{equation*}
    
    \pagebreak[3]
	\noindent Figure \ref{fig:ex:cl:ncl3unwanted} illustrates \( \clRevision \) graphically. We make two observations regarding \( \clRevision \):
	\begin{itemize}\setlength{\itemsep}{0pt}
		\item[]\hspace{-1.75ex}{\normalfont\textbf{Observation I.}} There  are situations where \( \clRevision \) yields an inconsistent belief set for a consistent formula (on a consistent belief set), e.g., we have \( \modelsOf{\Psi_{a}\clRevision \negOf{a}}=\emptyset \).
		\item[]\hspace{-1.75ex}{\normalfont\textbf{Observation II.}} There  are situations where \( \clRevision \) yields a consistent belief set for a consistent formula \( \alpha \)   (on an inconsistent belief set) and yields an inconsistent belief set for some consequences of \( \alpha \), e.g., we have \( \modelsOf{\Psi_{\bot}\clRevision a}=\{a\} \) and \( \modelsOf{\Psi_{\bot}\clRevision \top}=\emptyset \).
	\end{itemize}	
\end{example}
\noindent Indeed, we obtained the intended behaviour.
\begin{proposition}\label{prop:operatorviolationCLR3}
    The operator \( \clRevision \) from Example~\ref{ex:cl:ncl3unwanted} satisfies \eqref{pstl:CLR1}--\eqref{pstl:CLR6} except for \eqref{pstl:CLR3}.
\end{proposition}
\begin{proof}
    Violation of \eqref{pstl:CLR3} is given by Example~\ref{ex:cl:ncl3unwanted}.
    From the definition of \( \clRevision \) we obtain that \( \clRevision \) satisfies \eqref{pstl:CLR1}, \eqref{pstl:CLR2}, and \eqref{pstl:CLR4}.
    We show satisfaction of \eqref{pstl:CLR5} and \eqref{pstl:CLR6}:
    \begin{itemize}\setlength{\itemsep}{0pt}
        \item[]\hspace{-1.75ex}\emph{\eqref{pstl:CLR5}} Note that we have \( \alpha\in\beliefsOf{\Psi} \) for each \( \Psi\in\setAllES \) and for each \( \alpha\in\propLang \). Consequently, \( \clRevision \) satisfies \eqref{pstl:CLR5}.
        \item[]\hspace{-1.75ex}\emph{\eqref{pstl:CLR6}} Let \( \gamma=\alpha\lor\beta \). For \( \alpha\equiv\beta \) we obtain \( \modelsOf{\Psi\clRevision\gamma}=\modelsOf{\Psi\clRevision\alpha} =\modelsOf{\Psi\clRevision\beta}\) from \eqref{pstl:CLR4}. In the following we assume \( \modelsOf{\alpha}\neq\modelsOf{\beta} \). Observe that this implies \( \modelsOf{\gamma}\neq\emptyset \). Next, we consider two subcases for \( \Psi\in\setAllES \):
        \begin{itemize}\setlength{\itemsep}{0pt}
            \item[]\hspace{-1.75ex}\emph{\( \Psi=\Psi_{\bot} \).}  
            Observe that we have \( \modelsOf{\Psi\clRevision\varphi}=\modelsOf{\Psi}=\emptyset \) for all \( \varphi \) with \( \modelsOf{\varphi}\neq\{a\} \).
            Consequently, if \( a\notin\modelsOf{\gamma}  \), then we obtain \( \modelsOf{\Psi\clRevision\gamma}=\modelsOf{\Psi}=\modelsOf{\Psi\clRevision\alpha}=\modelsOf{\Psi\clRevision\beta}=\emptyset \). 
            If \( \modelsOf{\gamma}=\{a\} \), then we have \( \modelsOf{\Psi\clRevision\gamma}=\{a\} \) and we obtain from \( \modelsOf{\alpha}\neq\modelsOf{\beta} \) that either \( \modelsOf{\alpha}=\{a\} \) or \( \modelsOf{\beta}=\{a\} \).
            Thus, we obtain either \( \modelsOf{\Psi\clRevision\gamma}=\modelsOf{\Psi\clRevision\alpha} \) or \( \modelsOf{\Psi\clRevision\gamma}=\modelsOf{\Psi\clRevision\beta} \) by \eqref{pstl:CLR4}.
            We consider the remaining case of \( \{a\}\subsetneq\modelsOf{\gamma}  \). Then we  have  \( \modelsOf{\Psi\clRevision\gamma}=\emptyset \).
            From \( \modelsOf{\alpha}\neq\modelsOf{\beta} \) we obtain  that \( \modelsOf{\alpha}\neq\{a\} \) or \( \modelsOf{\beta}\neq\{a\} \) holds.
            Thus, we obtain either \( \modelsOf{\Psi\clRevision\gamma}=\modelsOf{\Psi\clRevision\alpha} \) or \( \modelsOf{\Psi\clRevision\gamma}=\modelsOf{\Psi\clRevision\beta} \).
            
            \item[]\hspace{-1.75ex}\emph{\( \Psi=\Psi_{a} \).} 
            Observe that we have \( \modelsOf{\Psi\clRevision\varphi}=\modelsOf{\Psi}=\{a\} \) for all \( \varphi \) with \( \modelsOf{\varphi}\not\subseteq\{\nAtom{a}\} \).
            Consequently, if \( \nAtom{a}\notin\modelsOf{\gamma}  \), then we obtain \( \modelsOf{\Psi\clRevision\gamma}=\modelsOf{\Psi}=\modelsOf{\Psi\clRevision\alpha}=\modelsOf{\Psi\clRevision\beta}=\emptyset \). 
            If \( \modelsOf{\gamma}=\{\nAtom{a}\} \), then we have \( \modelsOf{\Psi\clRevision\gamma}=\emptyset \) and we obtain from \( \modelsOf{\alpha}\neq\modelsOf{\beta} \) that either \( \modelsOf{\alpha}=\{\nAtom{a}\} \) or \( \modelsOf{\beta}=\{\nAtom{a}\} \).
            Thus, we obtain either \( \modelsOf{\Psi\clRevision\gamma}=\modelsOf{\Psi\clRevision\alpha} \) or \( \modelsOf{\Psi\clRevision\gamma}=\modelsOf{\Psi\clRevision\beta} \) by \eqref{pstl:CLR4}.
            We consider the remaining case of \( \{\nAtom{a}\}\subsetneq\modelsOf{\gamma}  \). Then we have  \( \modelsOf{\Psi\clRevision\gamma}=\{a\} \).
            From \( \modelsOf{\alpha}\neq\modelsOf{\beta} \) we obtain  that \( \modelsOf{\alpha}\neq\{\nAtom{a}\} \) or \( \modelsOf{\beta}\neq\{\nAtom{a}\} \) holds.
            Thus, we obtain either \( \modelsOf{\Psi\clRevision\gamma}=\modelsOf{\Psi\clRevision\alpha} \) or \( \modelsOf{\Psi\clRevision\gamma}=\modelsOf{\Psi\clRevision\beta} \).
        \end{itemize}
    \end{itemize}
    In summary, \( \clRevision \)  satisfies  \eqref{pstl:CLR1}--\eqref{pstl:CLR6}  except for \eqref{pstl:CLR3}.
\end{proof}
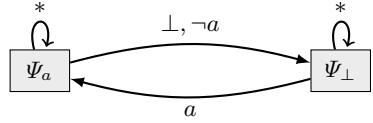
\begin{figure}[t]
	\centering\begin{tikzpicture}
            \tikzstyle{esnode}=[draw,fill=gray!15,text=black,minimum width=2.5em,minimum height=1.75em]
            \tikzstyle{myedge}=[->,thick,-{latex}]
            
            \tikzstyle{myedgenode}=[sloped]

            \node (psiA) [esnode,anchor=west] at (0,0) {\(\Psi_{a}\)};
            \node (psiB) [esnode,anchor=west] at (4,0) {\(\Psi_{\bot}\)};
            
            \draw (psiA) edge [loop above,thick] node {$*$}  (psiA) ;
            \draw (psiB) edge [loop above,thick] node {$*$}  (psiA) ;
            
            \draw (psiA) edge [myedge,bend left=15] node [myedgenode,above] {$\bot,\neg a$}  (psiB);
            \draw (psiB) edge [myedge,bend left=15] node [myedgenode,below] {$a$}  (psiA);
    \end{tikzpicture}
	\caption{Graphical representation of the operator \( \clRevision \) given in Example \ref{ex:cl:ncl3unwanted}.}\label{fig:ex:cl:ncl3unwanted}
\end{figure}
\subsection{Defining Extended Credibility-Limited Revision}
For extended credibility-limited revision we replace \eqref{pstl:CLR3} by postulates that prevent the behaviour given in Observation I and Observation II in Example \ref{ex:cl:ncl3unwanted}. 
The first postulate is
\begin{description}
	\item[\normalfont(\textlabel{CL3wcp}{pstl:CLR3wcp})] If \( \beliefsOf{\Psi\clRevision\alpha} \) is inconsistent, then \( \beliefsOf{\Psi} \) or \( \alpha \) is inconsistent.
\end{description}
which is already known in its contrapositive formulation,
\begin{description}
	\item[\normalfont(\textlabel{WCP}{pstl:WCP})] If \( \beliefsOf{\Psi} \) and \( \alpha \)  are consistent, then \( \beliefsOf{\Psi\clRevision\alpha} \) is consistent.
\end{description}
as \emph{weak consistency preservation} \cite{KS_HanssonFermeCantwellFalappa2001,KS_KatsunoMendelzon1991}.
The postulate \eqref{pstl:CLR3wcp} states that  the inconsistency of the result of a change on \( \Psi \) by \( \alpha \) is rooted in inconsistency of either \( \beliefsOf{\Psi} \) or \( \alpha \).
Moreover, we will assume satisfaction of the following postulate:
\begin{description}
	\item[\normalfont(\textlabel{CL3u}{pstl:CLR3w})] If $ \beliefsOf{\Psi\clRevision\alpha} $ is consistent and \( \alpha\models\beta \), then $ \beliefsOf{\Psi\clRevision\beta} $ is consistent.
\end{description}
The postulate \eqref{pstl:CLR3w} states that the consistency of a change on \( \Psi \) by \( \alpha \) is inherited \enquote{upward} to all changes on \( \Psi \) by consequences of \( \alpha \).
Regarding our observations in Example \ref{ex:cl:ncl3unwanted}: the postulate \eqref{pstl:CLR3wcp} prevents situations like in Observation~I, and the postulate \eqref{pstl:CLR3w} rules out situations mentioned in Observation~II of Example \ref{ex:cl:ncl3unwanted}.
Considering  \eqref{pstl:CLR3}, \eqref{pstl:CLR3w}, and \eqref{pstl:CLR3wcp} yields directly the interrelation of these postulates.
\begin{proposition}
    Let \( \epistemicSpace=\tuple{\setAllES,\beliefSymb} \) be an epistemic space and \( \circ \) be a belief change operator for \( \epistemicSpace \). If \( \circ \) satisfies \eqref{pstl:CLR3}, then \( \circ \) satisfies \eqref{pstl:CLR3w} and \eqref{pstl:CLR3wcp}.
\end{proposition}
\begin{proof}
    Suppose that \( \circ \) satisfies \eqref{pstl:CLR3}.
    Then, the antecedent of  \eqref{pstl:CLR3wcp} is never fulfilled, and hence,  \eqref{pstl:CLR3wcp}  is always satisfied by \( \circ \).
    For \eqref{pstl:CLR3w}, observe that the consequent of \eqref{pstl:CLR3w} is always fulfilled by \( \circ \). Consequently,  \eqref{pstl:CLR3w} is always satisfied by \( \circ \).
\end{proof}

Given these postulates, we define extended credibility-limited revision operators for epistemic spaces in the following as operators that satisfy \eqref{pstl:CLR1}, \eqref{pstl:CLR2}, \eqref{pstl:CLR3wcp}, \eqref{pstl:CLR3w} and \eqref{pstl:CLR4}--\eqref{pstl:CLR6}. For the sake of clarity, we give this set of postulates its own naming.
\pagebreak[3]
\begin{definition}[Extended Credibility-Limited Revision]\label{def:clr}
	Let \( \epistemicSpace=\tuple{\setAllES,\beliefSymb} \) be an epistemic space.
	A belief change operator $ \clRevision $ for \( \epistemicSpace \) is an \emph{extended credibility-limited revision operator for \( \epistemicSpace \)} if $ \clRevision $ satisfies:
	\begin{description}
		\item[\normalfont(\textlabel{ECL1}{pstl:GCLR1})] $ \alpha\in \beliefsOf{\Psi\clRevision\alpha} \ksOR \beliefsOf{\Psi\clRevision\alpha} = \beliefsOf{\Psi} $
		\item[\normalfont(\textlabel{ECL2}{pstl:GCLR2})] \(   \beliefsOf{\Psi \clRevision \alpha} \,{=}\, {\beliefsOf{\Psi} \,{+}\, \alpha } \text{ if } \beliefsOf{\Psi} \,{+}\, \alpha  \text{ is consistent } \)
		\item[\normalfont(\textlabel{ECL3}{pstl:GCLR3})] If \( \beliefsOf{\Psi\clRevision\alpha} \) is inconsistent, then \( \beliefsOf{\Psi} \) or \( \alpha \) is inconsistent
		\item[\normalfont(\textlabel{ECL4}{pstl:GCLR4})] If $ \beliefsOf{\Psi\clRevision\alpha} $ is consistent and \( \alpha\models\beta \), then $ \beliefsOf{\Psi\clRevision\beta} $ is consistent
		\item[\normalfont(\textlabel{ECL5}{pstl:GCLR5})] If $ \alpha\equiv\beta $, then $ \beliefsOf{\Psi\clRevision\alpha} = \beliefsOf{\Psi\clRevision\beta} $
		\item[\normalfont(\textlabel{ECL6}{pstl:GCLR6})] If $ \alpha \! \in\! \beliefsOf{\Psi\clRevision\alpha} $ and $ \alpha\models\beta $, then $ \beta \in \beliefsOf{\Psi\clRevision\beta} $
		\item[\normalfont(\textlabel{ECL7}{pstl:GCLR7})] $ \beliefsOf{\Psi\!\clRevision\!(\alpha\lor\beta)}\! =\!\begin{cases}
			\beliefsOf{\Psi\clRevision\alpha} \text{ or} \\
			\beliefsOf{\Psi\clRevision\beta} \text{ or} \\
			\beliefsOf{\Psi\clRevision\alpha} \cap \beliefsOf{\Psi\clRevision\beta}
		\end{cases}  $
	\end{description}
\end{definition}
\subsection{Semantic Characterization}
Next, we characterize extended credibility-limited revision operators semantically.
Booth et al. \cite{KS_BoothFermeKoniecznyPerez2012} proposed to use faithful assignments to capture the class of credibility-limited revision operators.
In the following, we present an extended version of their assignments, which are meant to capture extended credibility-limited revision operators.

\begin{definition}[(Extended) Credibility-Limited Assignment]\label{def:clr-assignment}
    Let \( \epistemicSpace=\tuple{\setAllES,\beliefSymb} \) be an epistemic space.
	A function $ \Psi \mapsto (\preceq_\Psi,C_\Psi,b_\Psi) $ is called an (extended) \emph{credibility-limited assignment for \( \epistemicSpace \)} if $ C_\Psi \subseteq \Omega $ is a set of interpretations with $ \modelsOf{\Psi} \subseteq C_\Psi  $, and $ \preceq_\Psi$ is a total preorder over $  C_\Psi $, and \( b_\Psi \in \{\top,\bot\}   \)  for all \( \Psi\in \setAllES \) such that the following~holds:
	\begin{description}
		\item[\normalfont(\textlabel{CLA\( {}_\bot \)}{pstl:CLAbot})] If \( b_\Psi = \bot \), then \( C_\Psi=\Omega \).
	\end{description}
	\end{definition}	
	(Extended) credibility-limited assignments carry two kinds of information. 
	First, \( C_\Psi\) describes semantically all consistent beliefs denoted as credible and \( b_\Psi  \) represents whether an inconsistent formula is considered as credible or not. 
	Note that \( b_\Psi \) is an extension to the assignments considered by Booth et al. \cite{KS_BoothFermeKoniecznyPerez2012}. 
	Second, the total preorder \( \preceq_{\Psi} \) serves the same purpose as in Katsuno-Mendelzon characterzation of revision \cite{KS_KatsunoMendelzon1992}; representing the preferences of the agent. 
	Note that \( \preceq_{\Psi} \) might be a relation over a strict subset of \( \Omega \). 

	\begin{definition}\label{def:clr-assignment_ff} 
        Let \( \epistemicSpace=\tuple{\setAllES,\beliefSymb} \) be an epistemic space.
		A credibility-limited assignment $ \Psi \mapsto (\preceq_\Psi,C_\Psi,b_\Psi) $ for \( \epistemicSpace \) is called \emph{faithful} if the following holds:
		\begin{description}
			\item[\normalfont(\textlabel{CLFA1}{pstl:CLFA1})] \( \text{If } \omega_1 \in \modelsOf{\Psi} \ksAND \omega_2 \in \modelsOf{\Psi} \ksTHEN \omega_1 \simeq_\Psi \omega_2 \)
			\item[\normalfont(\textlabel{CLFA2}{pstl:CLFA2})] \( \text{If } \omega_1 \in \modelsOf{\Psi} \ksAND \omega_2 \notin\modelsOf{\Psi} \ksTHEN \omega_1 <_\Psi \omega_2 \)
		\end{description}
	\end{definition}	
	We connect credibility-limited assignments with belief change operators by the following notion of compatibility~\cite{KS_FalakhRudolphSauerwald2022}.
	\begin{definition}\label{def:clrevision-compatible}
		A credibility-limited assignment \( \Psi \mapsto (\preceq_\Psi,C_\Psi,b_\Psi) \) is called \emph{(credibility-limited) revision-compatible} with a belief change operator \( \clRevision \) if the following holds:
		\begin{equation}
			\tag{revision-compatible}\label{eq:cl_revision}
			\modelsOf{\Psi \clRevision \alpha} = \begin{cases}
				\min(\modelsOf{\alpha},\preceq_\Psi) \!\!\! & \ksIF \modelsOf{\alpha} \cap C_\Psi \neq \emptyset \\
				\emptyset & \ksIF \modelsOf{\alpha}=\emptyset \ksAND b_\Psi=\bot\\
				\modelsOf{\Psi} & \ksOtherwise 
			\end{cases}
		\end{equation}
	\end{definition}	
	Given the notion of \hyperref[eq:cl_revision]{revision-compatibility}, we will now show that faithful credibility-limited assignments fully capture extended credible-limited revision operators for epistemic states. 
	\begin{theorem}\label{thm:credibilitylimited_characterization}
        Let \( \epistemicSpace=\tuple{\setAllES,\beliefSymb} \) be an epistemic space and let  $ \clRevision $ be a belief change operator for \( \epistemicSpace \).
		Then \( \clRevision \) is an extended credibility-limited revision operator for \( \epistemicSpace \) if and only if there is a faithful credibility-limited assignment $ \Psi \mapsto (\preceq_\Psi,C_\Psi, b_{\Psi}) $ that is \ref{eq:cl_revision} with \( \clRevision \).
	\end{theorem}
\begin{proof}[Proof (idea)]
    Overall, the proofs follows a similar structure as the proof for the semantic characterization of (non-extended) credibility-limited revision by Booth et al. \cite{KS_BoothFermeKoniecznyPerez2012}.
    Their proof is conceptually extended by dealing with inconsistency and adapted to deal with the two different postulates \eqref{pstl:CLR3w} and \eqref{pstl:CLR3wcp}. 
    For the \( \Rightarrow \)-direction, one has to give a construction of an faithful credibility-limited assignment $ \Psi \mapsto (\preceq_\Psi,C_\Psi, b_{\Psi}) $ that is \ref{eq:cl_revision} with \( \clRevision \).
    The construction used in the full proof works as follows. We set \( C_{\Psi} \) as follows
    \begin{align*}
        C_{\Psi}   & = \{  \omega \mid \modelsOf{\varphi_{\omega}} = \modelsOf{\Psi\clRevision\varphi_{\omega}} \}  ,  \tag{{see \text{\cite[Remark 1]{KS_BoothFermeKoniecznyPerez2012}}}}
    \end{align*}
    for each \( \Psi\in\setAllES \), where \( \varphi_\omega \) denotes a formula with \( \modelsOf{\varphi_\omega}=\{\omega\} \).
    If \( \modelsOf{\Psi}\neq\emptyset \) and  \( \bot\in\beliefsOf{\Psi\clRevision\bot} \), then we  set \( b_{\Psi}=\bot \); otherwise we set \( b_\Psi=\top \).
    For each \( \Psi\in\setAllES \) let \( {\preceq_{\Psi}}\subseteq C_\Psi\times C_\Psi \)   be the relation such that
    \begin{align*}
        \omega_1 \preceq_{\Psi} \omega_2 &   \text{ if and only if }  \omega_1\in \modelsOf{\Psi\clRevision\varphi_{\omega_1,\omega_2}}
    \end{align*}
    holds, where \( \varphi_{\omega_1,\omega_2} \) denotes a formula with \( \modelsOf{\varphi_{\omega_1,\omega_2}}=\{\omega_1,\omega_2\} \).
\end{proof}

\section{Example and Properties}
\label{sec:examples_properties}
	In the following, we consider an example for an extended credibility-limited revision operator and demonstrate the semantic characterization by Theorem~\ref{thm:credibilitylimited_characterization}.

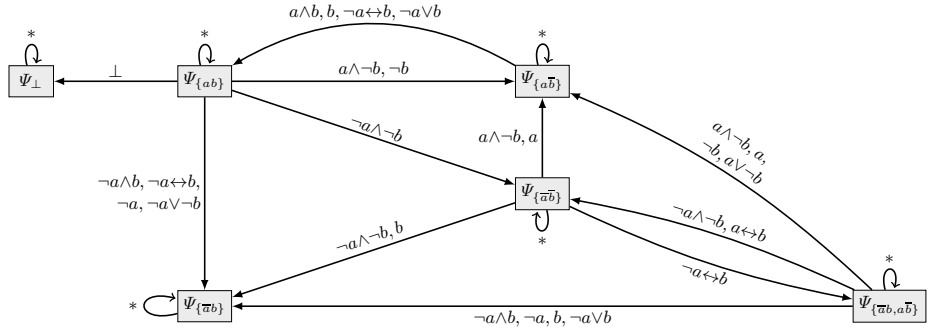
\begin{figure*}[t]
	\centering
	\resizebox{\textwidth}{!}{\begin{tikzpicture}
			\tikzstyle{esnode}=[draw,fill=gray!15,text=black,minimum width=2.5em,minimum height=1.75em]
			\tikzstyle{myedge}=[->,thick,-{latex}]

			\tikzstyle{myedgenode}=[inner sep=0.4ex]
			
			\node (psiB) [esnode,anchor=west] at (0,0) {\(\Psi_{\bot}\)};
			
			\node (psiAB) [esnode,anchor=west] at (3,0) {\(\Psi_{\{{a}{b}\}}\)};
			\node (psiAnB) [esnode,anchor=west] at (9,0) {\(\Psi_{\{{a}\nAtom{b}\}}\)};
			\node (psinAB) [esnode,anchor=west] at (3,-4) {\(\Psi_{\{\nAtom{a}{b}\}}\)};
			\node (psinAnB) [esnode,anchor=west] at (9,-2) {\(\Psi_{\{\nAtom{a}\nAtom{b}\}}\)};
			\node (psinAnBnAB) [esnode,anchor=west] at (15,-4) {\(\Psi_{\{\nAtom{a}{b},{a}\nAtom{b}\}}\)};
			
			\draw (psiAB) edge [myedge] node [myedgenode,above,sloped] {$a{\land}\neg b,\neg b$}  (psiAnB);
			\draw (psiAB) edge [myedge] node [myedgenode,left,text width=3cm,align=right] {$\neg a{\land} b, \neg a {\leftrightarrow} b,\linebreak\neg a,\neg a {\lor} \neg b$}  (psinAB);
			\draw (psiAB) edge [myedge] node [myedgenode,above,sloped] {$\neg a{\land}\neg b$}  (psinAnB);
			
			\draw (psinAnB) edge [myedge] node [myedgenode,above,sloped] {$\neg a{\land} \neg b,b$}  (psinAB);
			\draw (psinAnB) edge [myedge] node [myedgenode,left] {$a{\land} \neg b,a$}  (psiAnB);
			\draw (psinAnB) edge [myedge,bend right=7] node [myedgenode,below,sloped] {$\neg a {\leftrightarrow} b$}  (psinAnBnAB);
			
			\draw (psiAnB) edge [myedge,bend right] node [myedgenode,above,sloped,text width=3cm] {$a{\land} b, b,\neg a{\leftrightarrow} b, \neg a {\lor} b$}  (psiAB);
			
			\draw (psinAnBnAB) edge [myedge,bend right=7] node [myedgenode,above,sloped] {$\neg a {\land}\neg b,a{\leftrightarrow}b$}  (psinAnB);
			\draw (psinAnBnAB) edge [myedge,bend left=0] node [myedgenode,below,sloped] {$\neg a{{\land}} b,\neg a,b,\neg a {\lor}b $}  (psinAB);
			\draw (psinAnBnAB) edge [myedge,bend right=10] node [myedgenode,above,sloped,text width=1.5cm] {$a{\land}\neg b,a,\linebreak\neg b,a{\lor}\neg b$}  (psiAnB);
			
			\draw (psiAB) edge [myedge] node [myedgenode,above] {$\bot$}  (psiB);

			\draw (psiB) edge [loop above,thick] node {$*$}  (psiB) ;
			\draw (psiAB) edge [loop above,thick] node {$*$}  (psiAB) ;
			\draw (psiAnB) edge [loop above,thick] node {$*$}  (psiAnB) ;
			\draw (psinAB) edge [loop left,thick] node {$*$}  (psinAB) ;
			\draw (psinAnB) edge [loop below,thick] node {$*$}  (psinAnB) ;
			\draw (psinAnBnAB) edge [loop above,thick] node {$*$}  (psinAnBnAB) ;
	\end{tikzpicture}}
	\caption{Graphical representation of the extended credibility-limited revision operator \( \clRevision \) given in Example \ref{ex:genclr}.}\label{fig:ex:genclr}
\end{figure*}
\begin{example}\label{ex:genclr}
    Let \( \Sigma=\{a,b\} \) and let \( \epistemicSpace_\mathrm{ex}=\tuple{\setAllES,\beliefSymb} \) be the epistemic space where \( \setAllES=\{ \allowbreak\Psi_\bot,\allowbreak\Psi_{\{ab\}},\allowbreak\Psi_{\{\nAtom{a}b\}},\allowbreak\Psi_{\{a\nAtom{b}\}},\Psi_{\{\nAtom{a}\nAtom{b}\}}, \Psi_{\{\nAtom{a}b,a\nAtom{b}\}} \} \) is a set of epistemic states with:
    \begin{align*}
    	\modelsOf{\Psi_\bot}                        & = \emptyset                   & \modelsOf{\Psi_{\{\nAtom{a}b,a\nAtom{b}\}}} & =\{ \nAtom{a}b, a\nAtom{b} \} \\
    	\modelsOf{\Psi_{\{ab\}}}         & = \{ ab \}         & \modelsOf{\Psi_{\{a\nAtom{b}\}}}         & = \{ a\nAtom{b}\}          \\
    	 \modelsOf{\Psi_{\{\nAtom{a}b\}}} & = \{ \nAtom{a}b \} & \modelsOf{\Psi_{\{\nAtom{a}\nAtom{b}\}}} & = \{ \nAtom{a}\nAtom{b} \}
    \end{align*}
    In the following, we obtain an extended credibility-limited revision operator \( \clRevision \) for \( \epistemicSpace_\mathrm{ex} \) by specifying a faithful credibility-limited assignment that is \ref{eq:cl_revision} with \( \circ \).
    We use the following linear order \( \ll \) on \( \Omega \):
    \begin{equation*}
        ab \ \ \ll\ \ \nAtom{a}b \ \ \ll\ \  a\nAtom{b} \ \ \ll\ \  \nAtom{a}\nAtom{b}
    \end{equation*}
    We specify \( \Psi \mapsto (\preceq_\Psi,C_\Psi, b_{\Psi})  \) stepwise. We start by providing \( C_\Psi \) for each \( \Psi\in\setAllES \), which encodes semantically the set of those formulas that are considered as credible:
    \begin{align*}
        C_{\Psi_\bot}             & = \emptyset        & C_{\Psi_{\{\nAtom{a}b,a\nAtom{b}\}}} & =\{ \nAtom{a}b, a\nAtom{b},\nAtom{a}\nAtom{b} \}  \\
        C_{\Psi_{\{ab\}}}         & = \Omega           & C_{\Psi_{\{a\nAtom{b}\}}}            & = \{ ab, a\nAtom{b}\}                             \\
        C_{\Psi_{\{\nAtom{a}b\}}} & = \{ \nAtom{a}b \} & C_{\Psi_{\{\nAtom{a}\nAtom{b}\}}}    & = \{ \nAtom{a}b ,a\nAtom{b},\nAtom{a}\nAtom{a} \}
    \end{align*}
    We set \( b_\Psi=\top \) for each \( \Psi\in\setAllES\setminus\{ \Psi_{\{ab\}} \} \), and set \( b_{\Psi_{\{ab\}}}=\bot \). Meaning
    For each \( \Psi\in\setAllES\setminus\{ \Psi_{\{\nAtom{a}\nAtom{b}\}} \} \) we set \( {\preceq_{\Psi}} \subseteq (C_\Psi\times C_\Psi) \):
    \begin{equation*}
        {\preceq_{\Psi}} = \left(({\ll}\cap (C_\Psi\times C_\Psi)) \setminus (C_\Psi\times \modelsOf{\Psi})\right) \cup (\modelsOf{\Psi}\times C_\Psi),
    \end{equation*}
    i.e., \( {\preceq_{\Psi}} \) is the total preorder on \( C_\Psi \) such that \( {\minOf{C_\Psi}{\preceq_{\Psi}}}=\modelsOf{\Psi} \) and the remaining elements in \( {C_\Psi\setminus\modelsOf{\Psi}} \) are ordered according to \( {\ll} \).
    For \( \Psi_{\{\nAtom{a}\nAtom{b}\}} \), we specify \( {\preceq_{\Psi_{\{\nAtom{a}\nAtom{b}\}}}} \subseteq (C_{\Psi_{\{\nAtom{a}\nAtom{b}\}}}\times C_{\Psi_{\{\nAtom{a}\nAtom{b}\}}}) \) as follows:
    \begin{align*}
        \nAtom{a}\nAtom{b} & \preceq_{\Psi_{\{\nAtom{a}\nAtom{b}\}}}  \nAtom{a}{b}       & \nAtom{a}{b} & \preceq_{\Psi_{\{\nAtom{a}\nAtom{b}\}}} \nAtom{a}{b} & {a}\nAtom{b} & \preceq_{\Psi_{\{\nAtom{a}\nAtom{b}\}}} \nAtom{a}{b} \\
        \nAtom{a}\nAtom{b} & \preceq_{\Psi_{\{\nAtom{a}\nAtom{b}\}}}  {a}\nAtom{b}       & \nAtom{a}{b} & \preceq_{\Psi_{\{\nAtom{a}\nAtom{b}\}}} {a}\nAtom{b} & {a}\nAtom{b} & \preceq_{\Psi_{\{\nAtom{a}\nAtom{b}\}}} {a}\nAtom{b} \\
        \nAtom{a}\nAtom{b} & \preceq_{\Psi_{\{\nAtom{a}\nAtom{b}\}}}  \nAtom{a}\nAtom{b} &           &           &              &
    \end{align*}
    Because \eqref{pstl:CLAbot}, \eqref{pstl:CLFA1}, and \eqref{pstl:CLFA2} are satisfied, \( \Psi \mapsto (\preceq_\Psi,C_\Psi, b_{\Psi})  \) is a faithful credibility-limited assignment.
    A belief change operator \( \clRevision \) for \( \epistemicSpace_\mathrm{ex} \) that is \ref{eq:cl_revision} with \( \Psi \mapsto (\preceq_\Psi,C_\Psi, b_{\Psi})  \) is then:
    \begin{equation*}
        \Psi \clRevision \alpha = \begin{cases}
            \Psi_{\modelsOf{\Psi}\cap \modelsOf{\alpha}} & \ksIF \modelsOf{\Psi} \cap \modelsOf{\alpha} \neq \emptyset                                                    \\
            \Psi_\bot                                    & \ksIF \modelsOf{\Psi} = \emptyset \ksAND \Psi = \Psi_{\{ab\}}                                                               \\
            \Psi_{\{ \nAtom{a}{b},{a}\nAtom{b} \}}       & \ksIF \{ \nAtom{a}{b},{a}\nAtom{b} \}\subseteq  \modelsOf{\alpha} \ksAND \Psi=\Psi_{\{\nAtom{a}\nAtom{b}\}} \\
            \Psi_{\minOf{\modelsOf{\alpha}}{\ll}}        & \ksOtherwise
        \end{cases}
    \end{equation*}    
    By Theorem~\ref{thm:credibilitylimited_characterization}, we obtain that \( \clRevision \) is an extended credibility-limited revision operator.
    A graphical representation of this operator is given in Figure \ref{fig:ex:genclr}.
\end{example}

Note that \( \clRevision \) in Example \ref{ex:genclr} has properties that AGM revision operators do not have.
The beliefs accepted for revision are not the full language \( \propLang \).
The selection of beliefs accepted for revision is done individually for each epistemic state. 
Inconsistent beliefs are only accepted for revision in selected epistemic states.
Moreover, \( \clRevision \) in Example \ref{ex:genclr} demonstrates that in contrast to the credibility-limited revision operators considered by Booth et al. \cite{KS_BoothFermeKoniecznyPerez2012} (cf. Section~\ref{sec:clr_subsec}), extended credibility-limited revision operators, as defined in Definition~\ref{def:clr},  are able to deal with inconsistent input and with inconsistent epistemic states, and therefore make use of the full unrestricted framework of belief change for epistemic spaces, as introduced in Section~\ref{sec:es_subsec}. 

The following proposition points out that our generalization approach is successful in the sense that every AGM revision operator for epistemic states is indeed an extended credibility-limited revision operator in the sense of Definition~\ref{def:clr}.
\begin{proposition}\label{prop:es_agm_is_ecl}
    Let \( \epistemicSpace=\tuple{\setAllES,\beliefSymb} \) be an epistemic space and let \( \revision \) be a belief change operator for \( \epistemicSpace \). %
    The operator \( \revision \) is an AGM revision operator for \( \epistemicSpace \) if and only if \( \revision \) is an extended credibility-limited revision operator for \( \epistemicSpace \) which is revision-compatible with some faithful credibility-limited assignment \( \Psi \mapsto (\preceq_{\Psi},C_{\Psi}, b_{\Psi}) \) where \( C_{\Psi}=\Omega \) and \( b_{\Psi}=\bot \) holds for each \( \Psi\in\setAllES \).
\end{proposition}
\begin{proof}[Proof (idea)]We consider both directions of the claim independently.
    \begin{itemize}%
        \item[]\hspace{-1.25em}{\( \Rightarrow \)} Suppose that \( \revision \) is an AGM revision operator for \( \epistemicSpace \).
    We use the credibility-limited assignment \( {\Psi \mapsto (\preceq_{\Psi},C_{\Psi}, b_{\Psi})} \) given by:
    \begin{align*}
       C_{\Psi} & = \Omega & b_{\Psi} & = \bot &  \omega_1 \preceq_{\Psi} \omega_2 & \text{ if } \omega_1\in \modelsOf{\Psi\revision\varphi_{\omega_1,\omega_2}}  
    \end{align*}    
    The proof by Darwiche and Pearl \cite[Thm. 9]{KS_DarwichePearl1997} yields that \( {\Psi \mapsto (\preceq_{\Psi},C_{\Psi}, b_{\Psi})} \) is faithful and \ref{eq:cl_revision} with \( \clRevision \).

\item[]\hspace{-1.25em}{\( \Leftarrow \)}~Suppose that \( \revision \) is an extended credibility-limited revision operator for \( \epistemicSpace \) and \( \Psi \mapsto (\preceq_{\Psi},C_{\Psi}, b_{\Psi}) \) is as given above.
We obtain that \( \revision \) satisfies \eqref{pstl:R2} and \eqref{pstl:R4}, because \eqref{pstl:R2} coincides with \eqref{pstl:GCLR1} and \eqref{pstl:R4} coincides with \eqref{pstl:GCLR5}.
Because \( C_{\Psi}=\Omega \) and \( b_{\Psi}=\bot \) holds, by considering the \hyperref[eq:cl_revision]{revision-compatibility} one sees easily that \eqref{pstl:R1} and \eqref{pstl:R3} are satisfied by \( \revision \).
To see that \eqref{pstl:R5} and \eqref{pstl:R6} are satisfies by \( \revision \), use that \eqref{pstl:GCLR7} is equivalent to \eqref{pstl:R5} and \eqref{pstl:R6} whenever \eqref{pstl:R1}--\eqref{pstl:R4} are satisfied~\cite{KS_Gaerdenfors1988}.\qedhere
            \end{itemize}
\end{proof}

Next, we show that extended credibility-limited revision really extends credibility-limited revision for epistemic states as advertised in Section~\ref{sec:introduction}.
Therefore, we use Theorem~\ref{thm:credibilitylimited_characterization} to characterize operators that satisfy \eqref{pstl:CLR1}--\eqref{pstl:CLR6}, including \eqref{pstl:CLR3}, when there is no epistemic state with an inconsistent belief set (see \hyperref[pstl:globalconsistency]{global consistency}, defined on p.~\pageref{pstl:globalconsistency}). Note that this is close to the setting originally considered by Booth et al. \cite{KS_BoothFermeKoniecznyPerez2012}.
\begin{proposition}\label{prop:es_cl_is_ecl}
    Let \( \epistemicSpace=\tuple{\setAllES,\beliefSymb} \) be a \ref{pstl:globalconsistency} 
    epistemic space and let  $ \clRevision $ be a belief change operator for \( \epistemicSpace \). 
    The operator \( \clRevision \) satisfies \eqref{pstl:CLR1}--\eqref{pstl:CLR6} if and only if there is a faithful credibility-limited assignment $ \Psi \mapsto (\preceq_\Psi,C_\Psi, b_{\Psi}) $ that is \ref{eq:cl_revision} with \( \clRevision \) such that 
    \( b_\Psi=\top \) for each \( \Psi\in\setAllES \).
\end{proposition}
\begin{proof}We consider both directions independently.
    \begin{itemize}\setlength{\itemsep}{0pt}
\item[]\hspace{-1.25em}{\( \Rightarrow \)}
        If \( \clRevision \) satisfies \eqref{pstl:CLR1}--\eqref{pstl:CLR6}, then \eqref{pstl:GCLR1}--\eqref{pstl:GCLR7} are satisfied (as \eqref{pstl:GCLR3} and \eqref{pstl:GCLR4} are implied by \eqref{pstl:CLR3}).
        By Theorem~\ref{thm:credibilitylimited_characterization}, there exists some faithful credibility-limited assignment \( \Psi \mapsto (\preceq_{\Psi},C_{\Psi}, b_{\Psi}) \) that is \ref{eq:cl_revision} with \( \clRevision \). 
        Let \( \Psi\in\setAllES \) be an epistemic state. 
        From \eqref{pstl:CLR3} we obtain that \( \modelsOf{\Psi\clRevision\alpha}\neq\emptyset \) for each \( \alpha\in\propLang \).
        Consequently, we have \( b_\Psi=\top \), as otherwise we would obtain \( \modelsOf{\Psi\clRevision\bot}=\emptyset \) by \hyperref[eq:cl_revision]{revision-compatibility}.
\item[]\hspace{-1.25em}{\( \Leftarrow \)}
        Suppose there is a credibility-limited assignment $ \Psi \mapsto (\preceq_\Psi,C_\Psi, b_{\Psi}) $ that is \ref{eq:cl_revision} with \( \clRevision \) such that  \( b_\Psi=\top \) for each \( \Psi\in\setAllES \).
        By Theorem~\ref{thm:credibilitylimited_characterization}, we obtain that \( \clRevision \) satisfies \eqref{pstl:CLR1}, \eqref{pstl:CLR2}, and \eqref{pstl:CLR4}--\eqref{pstl:CLR6}. For satisfaction of \eqref{pstl:CLR3} observe that by \hyperref[eq:cl_revision]{revision-compatibility} we obtain \( \modelsOf{\Psi\clRevision\bot}\neq\emptyset \) due to \( b_\Psi=\top \) for each \( \Psi\in\setAllES \).
        For all consistent formulas \( \alpha \) we have \( \modelsOf{\Psi\clRevision\alpha}\neq\emptyset \) due to the \hyperref[pstl:globalconsistency]{global consistency} of~\( \setAllES \). \qedhere
    \end{itemize}
\end{proof}
With \( \textsf{ECLRev}(\epistemicSpace) \) we denote the set off all extended credibility-limited revision operators for \( \epistemicSpace \), i.e.,
    \( \textsf{ECLRev}(\epistemicSpace) \,{=}\, \{\     {\clRevision} :\setAllES\times \logicLanguage \to \setAllES \mid {\clRevision} \text{ satisfies } \allowbreak \text{\eqref{pstl:GCLR1}--}\allowbreak\text{\eqref{pstl:GCLR7}}\,\} \).
The next proposition summarizes the interrelation between  the class of extended credibility-limited revisions operators, the class of credibility-limited revision operators and the class of AGM revision operators in the framework of epistemic spaces.
\begin{proposition}
    For each epistemic space \( \epistemicSpace \) holds:
    \begin{align*}
        \textsf{AGMRev}(\epistemicSpace)& \subseteq \textsf{ECLRev}(\epistemicSpace) \\
        \textsf{CLRev}(\epistemicSpace) & \subseteq \textsf{ECLRev}(\epistemicSpace) 
    \end{align*}
\end{proposition}
\begin{proof}
    The statement \( \textsf{AGMRev}(\epistemicSpace) \subseteq \textsf{ECLRev}(\epistemicSpace) \) is a direct consequence of Proposition~\ref{prop:es_agm_is_ecl}. 
    From Proposition~\ref{prop:es_cl_is_ecl}, we obtain \( \textsf{CLRev}(\epistemicSpace) \subseteq \textsf{ECLRev}(\epistemicSpace) \) whenever \( \epistemicSpace \) is a \ref{pstl:globalconsistency} epistemic spaces. In all cases where \( \epistemicSpace \) is not a \ref{pstl:globalconsistency} epistemic spaces, we obtain \( \textsf{CLRev}(\epistemicSpace)=\emptyset \) from Proposition~\ref{prop:no_cl_inunrestricted}.
\end{proof} 

\section{Conclusion}
\label{sec:conclusion}
In this paper, we considered belief changes in the unrestricted framework of epistemic spaces, which means inconsistent beliefs are permitted.
Credibility-limited revision as defined by Booth et al. \cite{KS_BoothFermeKoniecznyPerez2012} does not extend well to this unrestricted setting, as AGM revision operators are not included and no operators exist when an epistemic state is present that has inconsistent beliefs.
Extended credibility-limited revision operators are introduced, and we show that this class of operators deals with the before-mentioned problems. 
All AGM revision operators are also extended credibility-limited revision operators and extended credibility-limited revision operators do exists for epistemic spaces with inconsistent epistemic states.
Furthermore, a semantic characterization of extended credibility-limited revision is presented.
The approach here might serve as a prototype of how to deal with inconsistent beliefs in the framework of epistemic spaces. Especially, when considered other approach to belief change, e.g., like update \cite{KS_KatsunoMendelzon1991,KS_FermeGoncalves2023} and other kinds of non-prioritized belief change \cite{KS_FermeHansson2001,KS_Hansson1999b,KS_FermeHansson1999}, in the framework of epistemic spaces.

Finally, I like to remark that, independently, Grimaldi, Martinez and Rodriguez \cite{KS_GrimaldiMartinezRodriguez2024}, made a similar approach to extending credibility-limited revision, that also uses \eqref{pstl:WCP}, but does not use \eqref{pstl:CLR3w} to deal with inconsistent belief sets. A comparison of both approaches could be insightful. 
\begin{credits}
    \subsubsection{\ackname}
I thank the 
anonymous 
reviewers for their valuable hints and comments that helped me to improve this paper.
The research reported here was partially supported by the Deutsche Forschungsgemeinschaft (DFG, grant 465447331, project “Explainable Belief Merging”, EBM).
\end{credits}

\clearpage
\bibliographystyle{splncs04}
\bibliography{bibexport}

\clearpage
\appendix
\pagenumbering{Roman}\gdef\thesection{\Alph{section}} 
\makeatletter
\renewcommand\@seccntformat[1]{Appendix \csname the#1\endcsname.\hspace{0.5em}}
\makeatother
\section{Full Proof of Theorem~\ref{thm:credibilitylimited_characterization}}\label{adx:proofs}

\newenvironment{repeatprop}[1]{\smallskip\par\noindent\textbf{Proposition~\ref{#1}}.\itshape}{\par}
\newenvironment{repeatthm}[1]{\smallskip\par\noindent\textbf{Theorem~\ref{#1}}.\itshape}{\par}
\newenvironment{repeatcorollary}[1]{\smallskip\par\noindent\textbf{Corollary~\ref{#1}}.\itshape}{\par}

We will make use of the following fact.
\begin{lemma}\label{lem:sem_trichotonomy}
    Let $ {\preceq} $  be a total preorder on $ M $. For each $ X,Y\subseteq M $ it holds:
    \begin{equation*}
        {\min(X\cup Y, {\preceq} )} = \begin{cases}
            {\min(X, {\preceq} )} & \ksOR\\
            {\min(Y, {\preceq} )} & \ksOR\\
            {\min(X, {\preceq} )} \cup {\min(Y, {\preceq} )} & 
        \end{cases}
    \end{equation*}
\end{lemma}

\begin{repeatthm}{thm:credibilitylimited_characterization}
    Let \( \epistemicSpace=\tuple{\setAllES,\beliefSymb} \) be an epistemic space and let  $ \clRevision $ be a belief change operator for \( \epistemicSpace \).
    Then \( \clRevision \) is an extended credibility-limited revision operator for \( \epistemicSpace \) if and only if there is a faithful credibility-limited assignment $ \Psi \mapsto (\preceq_\Psi,C_\Psi, b_{\Psi}) $ that is \ref{eq:cl_revision} with \( \clRevision \).
\end{repeatthm}
	\begin{proof}We consider both directions of the claim independently.

		\emph{The \enquote{\( \Rightarrow \)}-direction.} Let \( \clRevision \) be a credibility-limited revision operator for \( \epistemicSpace \).
		We construct a mapping $ \Psi \mapsto (\preceq_\Psi,C_\Psi, b_{\Psi}) $.
		We set \( C_{\Psi} \) as follows
		\begin{align*}
			C_{\Psi}   & = \{  \omega \mid \modelsOf{\varphi_{\omega}} = \modelsOf{\Psi\clRevision\varphi_{\omega}} \}  ,  \tag{{see \text{\cite[Remark 1]{KS_BoothFermeKoniecznyPerez2012}}}}
		\end{align*}
		for each \( \Psi\in\setAllES \), where \( \varphi_\omega \) denotes a formula with \( \modelsOf{\varphi_\omega}=\{\omega\} \).
		If \( \modelsOf{\Psi}\neq\emptyset \) and  \( \bot\in\beliefsOf{\Psi\clRevision\bot} \), then we  set \( b_{\Psi}=\bot \); otherwise we set \( b_\Psi=\top \).
		For each \( \Psi\in\setAllES \) let \( {\preceq_{\Psi}}\subseteq C_\Psi\times C_\Psi \)   be the relation such that
		\begin{align*}
			\omega_1 \preceq_{\Psi} \omega_2 &   \text{ if and only if }  \omega_1\in \modelsOf{\Psi\clRevision\varphi_{\omega_1,\omega_2}}
		\end{align*}
		holds, where \( \varphi_{\omega_1,\omega_2} \) denotes a formula with \( \modelsOf{\varphi_{\omega_1,\omega_2}}=\{\omega_1,\omega_2\} \).
		Next, we show that $ \Psi \mapsto (\preceq_\Psi,C_\Psi, b_{\Psi}) $ is a credibility-limited assignment.
		\begin{itemize}\setlength{\itemsep}{0pt}
			\item[]\hspace{-1.75ex}\emph{\( \modelsOf{\Psi}\subseteq C_\Psi \).} Let \( \omega\in\modelsOf{\Psi} \) and \( \varphi_\omega \) such that \( \modelsOf{\varphi_{\omega}}=\{\omega\} \). Clearly, \( \varphi_{\omega} \) is a formula such that \( \beliefsOf{\Psi}\cup\{\varphi_{\omega}\} \) is consistent. 
			From \eqref{pstl:GCLR2} we obtain \( \modelsOf{\varphi_{\omega}}=\modelsOf{\Psi\clRevision\varphi_{\omega}} \). 
			Consequently, we obtain \( \omega\in C_\Psi \) from the definition of \( C_\Psi \). This shows \( \modelsOf{\Psi}\subseteq C_\Psi \).
			
			\item[]\hspace{-1.75ex}\emph{\( {\preceq_{\Psi}} \) is a total preorder.} Reflexivity is a direct consequence of totality, thus in the following we show only totality and transitivity of \( \preceq_{\Psi} \):

				\emph{Totality.} Let \( \omega_1,\omega_2\in C_\Psi \). 
					We show totality by contradiction. Therefore, assume \( \omega_1\not\preceq_{\Psi} \omega_2 \) and \( \omega_2 \not\preceq_{\Psi} \omega_1 \) in the following.
					From the definition of \( {\preceq_{\Psi}} \)  we obtain \( \omega_1,\omega_2\notin\modelsOf{\Psi\clRevision\varphi_{\omega_1,\omega_2}} \), where \( \varphi_{\omega_1,\omega_2} \) is a formula such that \( \modelsOf{\Psi\clRevision\varphi_{\omega_1,\omega_2}}=\{\omega_1,\omega_2\} \).
					From \eqref{pstl:GCLR5}, we obtain that \( \modelsOf{\Psi\clRevision\varphi_{\omega_1,\omega_2}}=\modelsOf{\Psi\clRevision(\varphi_{\omega_1}\lor\varphi_{\omega_2})} \) holds.
					By \eqref{pstl:GCLR7} we have that \( \modelsOf{\Psi\clRevision(\varphi_{\omega_1}\lor\varphi_{\omega_2})} \) is equivalent to either \( \modelsOf{\Psi\clRevision\varphi_{\omega_1}} \) or \( \modelsOf{\Psi\clRevision\varphi_{\omega_2}} \) or \( \modelsOf{\Psi\clRevision\varphi_{\omega_1}}\cup\modelsOf{\Psi\clRevision\varphi_{\omega_2}} \).
					From the definition of \( C_\Psi \) we obtain \( \modelsOf{\Psi\clRevision\varphi_{\omega_1}}=\{\omega_1\} \) and \( \modelsOf{\Psi\clRevision\varphi_{\omega_2}}=\{\omega_2\} \).
					Consequently, we obtain that \( \omega_1\in\modelsOf{\Psi\clRevision\varphi_{\omega_1,\omega_2}} \) or \( \omega_2\in\modelsOf{\Psi\clRevision\varphi_{\omega_1,\omega_2}} \) holds, which is a contradiction to our prior observation of \( \omega_1,\omega_2\notin\modelsOf{\Psi\clRevision\varphi_{\omega_1,\omega_2}} \).
				\emph{Transitivity.} Let \( \omega_1,\omega_2,\omega_3\in C_\Psi \). 
				We show transitivity by contradiction. Therefore, we assume \( \omega_1\preceq_{\Psi}\omega_2 \) and \( \omega_2\preceq_{\Psi}\omega_3 \) and \( \omega_1\not\preceq_{\Psi}\omega_3 \) in the following.
				The latter assumption and the definition of \( \preceq_\Psi \) yield \( {\omega_1\notin\modelsOf{\Psi\clRevision(\varphi_{\omega_1,\omega_3})}} \).
				By the definition of \( C_\Psi \), and using \eqref{pstl:GCLR5} and \eqref{pstl:GCLR7}, we obtain \( \modelsOf{\Psi\clRevision(\varphi_{\omega_1,\omega_3})}=\{\omega_3\} \).
				In the following, we consider the same cases as in \cite[p. 22]{KS_DarwichePearl1997}:
				\begin{itemize}\setlength{\itemsep}{0pt}
					\item[]\hspace{-1.75ex}\emph{\( \omega_1\in\modelsOf{\Psi} \).} Observe that \( \clRevision \) satisfies \eqref{pstl:GCLR2}, and thus we obtain the contradiction \( \omega_1\in\modelsOf{\Psi\clRevision\varphi_{\omega_1,\omega_3}} \). 

					\item[]\hspace{-1.75ex}\emph{\( \omega_1\notin\modelsOf{\Psi} \) and \( \omega_2\in\modelsOf{\Psi} \).} Observe that \( \clRevision \) satisfies \eqref{pstl:GCLR2}, and thus we have \( \omega_2\in\modelsOf{\Psi\clRevision\varphi_{\omega_1,\omega_2}} \) and \( \omega_1\notin\modelsOf{\Psi\clRevision\varphi_{\omega_1,\omega_2}} \). Thus, by the definition of \( \preceq_{\Psi} \), we obtain the contradiction \( \omega_1\not\preceq_{\Psi}\omega_2 \).
					
					\item[]\hspace{-1.75ex}\emph{\( \omega_1\notin\modelsOf{\Psi} \) and \( \omega_2\notin\modelsOf{\Psi} \).}
					In the following let \( \varphi_{\omega_1,\omega_2,\omega_3} \) be a formula such that \( \modelsOf{\varphi_{\omega_1,\omega_2,\omega_3}}=\{\omega_1,\omega_2,\omega_3\} \).
					Recall that by the definition of \( C_\Psi \) we have \( \modelsOf{\Psi\clRevision\varphi_{\omega}}=\{\omega\} \) for each \( \omega\in\{\omega_1,\omega_2,\omega_3\} \).
					We consider two subcases:

					\emph{The case of \( \modelsOf{\Psi\clRevision\varphi_{\omega_1,\omega_2,\omega_3}} = \{\omega_3\} \).} Using \eqref{pstl:GCLR5} we obtain \( \modelsOf{\Psi\clRevision\varphi_{\omega_1,\omega_2,\omega_3}}=\modelsOf{\Psi\clRevision(\varphi_{\omega_1}\lor\varphi_{\omega_2,\omega_3})} \).
					Because we have \( \omega_3\notin \modelsOf{\Psi\clRevision\varphi_{\omega_1}} \), we obtain \( \modelsOf{\Psi\clRevision\varphi_{\omega_1,\omega_2,\omega_3}}=\modelsOf{\Psi\clRevision\varphi_{\omega_2,\omega_3}} \) from \eqref{pstl:GCLR7}.
					Using the definition of \( \preceq_{\Psi} \) and \( \modelsOf{\Psi\clRevision\varphi_{\omega_1,\omega_2,\omega_3}}=\{\omega_3\} \) we obtain the contradiction \( \omega_2\not\preceq_{\Psi}\omega_3 \).

					\emph{The case of \( \modelsOf{\Psi\clRevision\varphi_{\omega_1,\omega_2,\omega_3}} \neq \{\omega_3\} \).} 
					By using \eqref{pstl:GCLR5} we obtain \( \modelsOf{\Psi\clRevision\varphi_{\omega_1,\omega_2,\omega_3}}=\modelsOf{\Psi\clRevision(\varphi_{\omega_1,\omega_2}\lor\varphi_{\omega_1,\omega_3})} \).
					Because we have \( \modelsOf{\Psi\clRevision\varphi_{\omega_1,\omega_3}}=\{\omega_3\} \) and \( \modelsOf{\Psi\clRevision\varphi_{\omega_1,\omega_2,\omega_3}} \neq \{\omega_3\} \), we obtain  \( \modelsOf{\Psi\clRevision\varphi_{\omega_1,\omega_2}} \subseteq \modelsOf{\Psi\clRevision\varphi_{\omega_1,\omega_2,\omega_3}} \) from \eqref{pstl:GCLR7}.
					By using the definition of \( \preceq_{\Psi} \) and \( \omega_1\preceq_{\Psi}\omega_2 \) we obtain \( \omega_1\in\modelsOf{\Psi\clRevision\varphi_{\omega_1,\omega_2}} \).
					Consequently, we have \( \omega_1\in \modelsOf{\Psi\clRevision\varphi_{\omega_1,\omega_2,\omega_3}} \).
					
					By using \eqref{pstl:GCLR5} again, we obtain \( \modelsOf{\Psi\clRevision\varphi_{\omega_1,\omega_2,\omega_3}}=\modelsOf{\Psi\clRevision(\varphi_{\omega_2}\lor\varphi_{\omega_1,\omega_3})} \).
					Because we have \( \omega_1\in \modelsOf{\Psi\clRevision\varphi_{\omega_1,\omega_2,\omega_3}} \) and \( \modelsOf{\Psi\clRevision\varphi_{\omega_2}}=\{\omega_2\} \), we obtain \( \modelsOf{\Psi\clRevision\varphi_{\omega_1,\omega_2,\omega_3}}\cap\{\omega_1,\omega_3\}=\modelsOf{\Psi\clRevision\varphi_{\omega_1,\omega_3}} \) from \eqref{pstl:GCLR7}.
					Consequently, we obtain \( \omega_1\in\modelsOf{\Psi\clRevision\varphi_{\omega_1,\omega_3}} \), which yields the contradiction  \( \omega_1\preceq_{\Psi} \omega_3 \).
				\end{itemize}
			\item[]\hspace{-1.75ex}\emph{Satisfaction of \eqref{pstl:CLAbot}.} Suppose that \( b_\Psi=\bot \) holds. Then, by the definition of \( b_\Psi \) we have \( \modelsOf{\Psi}\neq\emptyset \) and \( \bot\in\beliefsOf{\Psi\clRevision\bot} \). We show \( C_\Psi=\Omega  \) by contradiction and assume therefore the existence of an \( \omega\in\Omega \) such that \( \omega\notin C_\Psi \).
			Because \( \clRevision \) satisfies \eqref{pstl:GCLR6}, we obtain \( \modelsOf{\Psi\clRevision\varphi_{\omega}}\subseteq\{\omega\} \) from \( \bot\in\beliefsOf{\Psi\clRevision\bot} \) and \( \bot\models\varphi_{\omega} \).
			From \( \omega\notin C_\Psi \) and the definition of \( C_\Psi \) we  obtain that \( \modelsOf{\Psi\clRevision\varphi_{\omega}}\neq\{\omega\} \) holds.
			By these observations, \( \modelsOf{\Psi\clRevision\varphi_{\omega}}=\emptyset \) remains as the only possibility.
			From \eqref{pstl:GCLR3} we obtain the contradiction that either \( \modelsOf{\Psi}=\emptyset \) or \( \modelsOf{\varphi_{\omega}}=\emptyset \) holds.
			Consequently, we have \( \omega \in C_\Psi \).
		\end{itemize}    
		In summary, $ \Psi \mapsto (\preceq_\Psi,C_\Psi, b_{\Psi}) $ is a credibility-limited assignment.
		We show that $ \Psi \mapsto (\preceq_\Psi,C_\Psi, b_{\Psi}) $ is faithful.
		Suppose that \( \modelsOf{\Psi}\neq\emptyset \) holds.
		\begin{itemize}\setlength{\itemsep}{0pt}
			\item[]\hspace{-1.75ex}\emph{\eqref{pstl:CLFA1}} Let \( \omega_1 \in \modelsOf{\Psi} \) and \( \omega_2 \in \modelsOf{\Psi} \). 
			From the satisfaction of \eqref{pstl:GCLR2} by \( \clRevision \) we obtain \( \omega_1,\omega_2\in\modelsOf{\Psi\clRevision\varphi_{\omega_1,\omega_2}} \). 
			Then, applying the definition of \( \preceq_{\Psi} \) yields \( \omega_1 \simeq_\Psi \omega_2 \), i.e., \(  \omega_1 \preceq_\Psi \omega_2 \) and  \(  \omega_2 \preceq_\Psi \omega_1 \).
			\item[]\hspace{-1.75ex}\emph{\eqref{pstl:CLFA2}} Let \( \omega_1 \in \modelsOf{\Psi} \) and \( \omega_2 \notin \modelsOf{\Psi} \). Using the satisfaction of \eqref{pstl:GCLR2} by \( \clRevision \) again, we obtain \( \omega_1\in\modelsOf{\Psi\clRevision\varphi_{\omega_1,\omega_2}} \) and \( \omega_2\notin\modelsOf{\Psi\clRevision\varphi_{\omega_1,\omega_2}} \). From the definition of \( \preceq_{\Psi} \) we obtain \( \omega_1 <_\Psi \omega_2 \).
		\end{itemize}

		Next, we show that $ \Psi \mapsto (\preceq_\Psi,C_\Psi, b_{\Psi}) $ is \ref{eq:cl_revision} with \( \clRevision \).
		Therefore, we consider four cases  in the following: the case of \( \modelsOf{\alpha}\cap C_\Psi\neq\emptyset \), the case of \( \modelsOf{\alpha}=\emptyset \) and \( b_\Psi=\bot \), the case of \( \modelsOf{\alpha}=\emptyset \) and \( b_\Psi=\top \), and the case of \( \modelsOf{\alpha}\neq\emptyset \) and \( {\modelsOf{\alpha}\cap C_\Psi=\emptyset} \).
		\begin{itemize}\setlength{\itemsep}{0pt}
			\item[]\hspace{-1.75ex}\emph{The case of \( \modelsOf{\alpha}\cap C_\Psi\neq\emptyset \).} 
			In this case, we directly obtain that \( \alpha \) is consistent.
			Moreover, from \( \modelsOf{\alpha}\cap C_\Psi\neq\emptyset \) we obtain an interpretation \( \omega\in \modelsOf{\alpha}\cap C_\Psi \).
			The definition of \( C_\Psi \) yields \( \modelsOf{\Psi\clRevision\varphi_\omega}=\{\omega\}  \).
			We obtain \( \modelsOf{\Psi\clRevision\alpha}\neq\emptyset \) from \eqref{pstl:GCLR4} and \( \varphi_\omega\models\alpha \), and from \eqref{pstl:GCLR6} that \( \modelsOf{\Psi\clRevision\alpha}\subseteq\modelsOf{\alpha} \) holds. 

			As the next step, we show \( \min(\modelsOf{\alpha},\preceq_{\Psi}) \subseteq \modelsOf{\Psi\clRevision\alpha} \) by contradiction.
			Suppose that there exists some \( \omega\in \min(\modelsOf{\alpha},\preceq_{\Psi})\) such that \( \omega\notin \modelsOf{\Psi\clRevision\alpha} \) holds.
			We consider two subcases.
			\begin{itemize}\setlength{\itemsep}{0pt}
				\item[]\hspace{-1.75ex}\emph{\( \omega\in\modelsOf{\Psi}  \).} Then we obtain \( \omega\in \modelsOf{\Psi\clRevision\alpha} \) from \eqref{pstl:GCLR2}.
				\item[]\hspace{-1.75ex}\emph{\( \omega\notin\modelsOf{\Psi}  \).}
				Because of \( \modelsOf{\Psi\clRevision\alpha}\neq\emptyset \) there exists some \( \omega'\in \modelsOf{\Psi\clRevision\alpha} \).
				
				We consider the case of \( \omega'\notin C_\Psi \).
				Consequently, we have \( \omega'\notin\modelsOf{\Psi} \).
				Now let \( \gamma \) be  a formula with \( \alpha\equiv \gamma\lor\varphi_{\omega'} \) and \( \omega'\not\models\gamma \).
				Using \eqref{pstl:GCLR5} we obtain that  \( \modelsOf{\Psi\clRevision\alpha}=\modelsOf{\Psi\clRevision(\gamma\lor\varphi_{\omega'})} \) holds.
				Note that by \eqref{pstl:GCLR1} we have \( \modelsOf{\Psi\clRevision\gamma}=\modelsOf{\Psi} \) or \( \modelsOf{\Psi\clRevision\gamma}\subseteq\modelsOf{\gamma} \). This implies that we have \( \omega'\notin \modelsOf{\Psi\clRevision\gamma} \).
				Consequently, we obtain \( \omega'\in\modelsOf{\Psi\clRevision\varphi_{\omega'}} \) from \eqref{pstl:GCLR7}.
				Using \( \omega'\notin\modelsOf{\Psi} \), \( \omega'\in\modelsOf{\Psi\clRevision\varphi_{\omega'}} \) and \eqref{pstl:GCLR1}, we obtain \( \modelsOf{\Psi\clRevision\varphi_{\omega'}}=\{\omega'\} \), which yields the contradiction \( \omega'\in C_\Psi \).
				
				We consider the case of \( \omega'\in C_\Psi \).
				Because of \( \omega\in \min(\modelsOf{\alpha},\preceq_{\Psi})\) we have \( \omega \preceq_{\Psi} \omega' \).
				Moreover, from faithfulness and \( \omega\notin\modelsOf{\Psi}  \), we obtain \( \omega'\notin\modelsOf{\Psi} \) from \( \omega\in \min(\modelsOf{\alpha},\preceq_{\Psi})\).
				From the definition of \( \preceq_{\Psi} \) we obtain \( \omega\in\modelsOf{\Psi\clRevision\varphi_{\omega,\omega'}} \).
				Now let \( \gamma \) be  a formula such that \( \omega,\omega'\notin\modelsOf{\gamma} \) and \( \alpha\equiv\gamma\lor\varphi_{\omega,\omega'} \).
				Note that by \eqref{pstl:GCLR1} we have \( \modelsOf{\Psi\clRevision\gamma}=\modelsOf{\Psi} \) or \( \modelsOf{\Psi\clRevision\gamma}\subseteq\modelsOf{\gamma} \).
				This together with \( \omega,\omega'\notin\modelsOf{\Psi} \) implies \( \omega'\notin \modelsOf{\Psi\clRevision\gamma} \).
				Using \eqref{pstl:GCLR5} we obtain that  \( \modelsOf{\Psi\clRevision\alpha}=\modelsOf{\Psi\clRevision(\gamma\lor\varphi_{\omega,\omega'})} \) holds.
				Because \( \clRevision \) satisfies \eqref{pstl:GCLR7} and \( \omega'\notin \modelsOf{\Psi\clRevision\gamma} \) holds, we have \( \modelsOf{\Psi\clRevision\alpha}=\modelsOf{\Psi\clRevision\varphi_{\omega,\omega'}}  \) or \( \modelsOf{\Psi\clRevision\alpha}=\modelsOf{\Psi\clRevision\gamma}\cup\modelsOf{\Psi\clRevision\varphi_{\omega,\omega'}}  \). In both cases we obtain \( \omega\in \modelsOf{\Psi\clRevision\alpha} \) from \( \omega\in\modelsOf{\Psi\clRevision\varphi_{\omega,\omega'}}  \).
			\end{itemize}
			In summary, \( \min(\modelsOf{\alpha},\preceq_{\Psi}) \subseteq \modelsOf{\Psi\clRevision\alpha} \) holds.
			
			We show by contradiction that \( \modelsOf{\Psi\clRevision\alpha}  \subseteq \min(\modelsOf{\alpha},\preceq_{\Psi}) \) holds.
			Therefore, suppose that there exists some \( \omega_1\in\modelsOf{\Psi\clRevision\alpha} \) such that \( \omega_1 \notin \min(\modelsOf{\alpha},\preceq_{\Psi}) \).
			The faithfulness of $ \Psi \mapsto (\preceq_\Psi,C_\Psi, b_{\Psi}) $ and \( \omega_1\in\modelsOf{\Psi}  \) together imply \( \omega_1\in  \min(\modelsOf{\alpha},\preceq_{\Psi})  \). 
			In the following we consider the remaining case of \( \omega_1 \notin \modelsOf{\Psi} \).
			From \( \modelsOf{\alpha}\cap C_\Psi\neq\emptyset \) we obtain that there exists some \( \omega_2\in\Omega \) with \( \omega_2\in\min(\modelsOf{\alpha},\preceq_{\Psi}) \).
			As shown before, we have \( \omega_2\in \modelsOf{\Psi\clRevision\alpha}\).
			Because \( \clRevision \) satisfies \eqref{pstl:GCLR2} and \(\omega_2\in \modelsOf{\Psi\clRevision\alpha}  \), we obtain \( \omega_2 \notin \modelsOf{\Psi} \) from \( \omega_1 \notin \modelsOf{\Psi} \).
			Now let \( \gamma \) be a formula such that \( \alpha\equiv\gamma\lor\varphi_{\omega_1,\omega_2} \) and \( \omega_1,\omega_2\notin\modelsOf{\gamma} \).
			Note that by \eqref{pstl:GCLR1} we have \( \modelsOf{\Psi\clRevision\gamma}=\modelsOf{\Psi} \) or \( \modelsOf{\Psi\clRevision\gamma}\subseteq\modelsOf{\gamma} \).
			This together with \( \omega_1,\omega_2\notin\modelsOf{\Psi} \) implies \( \omega_1,\omega_2\notin \modelsOf{\Psi\clRevision\gamma} \). 
			Using \eqref{pstl:GCLR5} we obtain that  \( \modelsOf{\Psi\clRevision\alpha}=\modelsOf{\Psi\clRevision(\gamma\lor\varphi_{\omega,\omega'})} \) holds.
			Because \( \clRevision \) satisfies \eqref{pstl:GCLR7} and \( \omega_1,\omega_2\notin \modelsOf{\Psi\clRevision\gamma} \) holds, we have \( \modelsOf{\Psi\clRevision\alpha}=\modelsOf{\Psi\clRevision\varphi_{\omega_1,\omega_2}}  \) or \( \modelsOf{\Psi\clRevision\alpha}=\modelsOf{\Psi\clRevision\gamma}\cup\modelsOf{\Psi\clRevision\varphi_{\omega_1,\omega_2}}  \).
			We obtain \( \modelsOf{\Psi\clRevision\alpha}\cap\{\omega_1,\omega_2\}=\modelsOf{\Psi\clRevision\varphi_{\omega_1,\omega_2}}=\{\omega_1,\omega_2\} \).
			Applying \eqref{pstl:GCLR5} and \eqref{pstl:GCLR7} again yields \( \modelsOf{\Psi\clRevision\varphi_{\omega_1,\omega_2}} = \modelsOf{\Psi\clRevision\varphi_{\omega_1}}\cup\modelsOf{\Psi\clRevision\varphi_{\omega_2}}   \).
			This together with \( \omega_1 \notin \modelsOf{\Psi} \) and \eqref{pstl:GCLR1} yields \( \modelsOf{\Psi\clRevision\varphi_{\omega_1}}=\{\omega_1\} \).
			By the definition of \( C_\Psi \), we obtain \( \omega_1\in C_\Psi\).
			Moreover, from \( \modelsOf{\Psi\clRevision\varphi_{\omega_1,\omega_2}}=\{\omega_1,\omega_2\} \) we obtain \( \omega_1\preceq_\Psi\omega_2 \) from the definition of \( \preceq_{\Psi} \).
			This last observation together with \( \omega_2\in\min(\modelsOf{\alpha},\preceq_{\Psi}) \) implies the contradiction \( \omega_1\in\min(\modelsOf{\alpha},\preceq_{\Psi}) \).

			\item[]\hspace{-1.75ex}\emph{The case of \( \modelsOf{\alpha}=\emptyset \) and \( b_\Psi=\bot \).} We show \( \alpha\in\beliefsOf{\Psi\clRevision\alpha} \), i..e, \( \modelsOf{\Psi\clRevision\alpha}=\emptyset \).
			From the definition of \( b_\Psi \) we obtain \( \bot\in\beliefsOf{\Psi\clRevision\bot} \) from \( b_\Psi=\bot \). 
			Because $\clRevision$ satisfies \eqref{pstl:GCLR5}, we obtain \( \alpha\in\beliefsOf{\Psi\clRevision\alpha} \).
			
			\item[]\hspace{-1.75ex}\emph{The case of \( \modelsOf{\alpha}=\emptyset \) and \( b_\Psi=\top \).} We show \( \beliefsOf{\Psi\clRevision\alpha}=\beliefsOf{\Psi} \), i.e., \( \modelsOf{\Psi\clRevision\alpha}=\modelsOf{\Psi} \).
			Consulting \eqref{pstl:GCLR1} yields that \( \alpha\in\beliefsOf{\Psi\clRevision\alpha} \) is the only non-trivial case to consider. 
			Because \( \clRevision \) satisfies \eqref{pstl:GCLR5} we obtain \( \bot\in\beliefsOf{\Psi\clRevision\bot} \).
			From the definition of \( b_\Psi \), and from \( b_\Psi=\top \) and \( \bot\in\beliefsOf{\Psi\clRevision\bot} \), we obtain \( \modelsOf{\Psi}=\emptyset \). 
			Consequently, we have \( \beliefsOf{\Psi\clRevision\alpha}=\beliefsOf{\Psi} \).
			
			\item[]\hspace{-1.75ex}\emph{The case of \( \modelsOf{\alpha}\neq\emptyset \) and \( \modelsOf{\alpha}\cap C_\Psi=\emptyset \).} We show \( \beliefsOf{\Psi\clRevision\alpha}=\beliefsOf{\Psi} \).
			Note that \( \modelsOf{\Psi}\subseteq C_\Psi \) holds, and therefore \( \modelsOf{\alpha}\cap C_\Psi=\emptyset \) implies \( \modelsOf{\alpha}\cap\modelsOf{\Psi}\neq\emptyset \).
			From the definition of \( C_\Psi \) we obtain \( \modelsOf{\Psi\clRevision\varphi_{\omega}}\neq\{\omega\} \) for each \( \omega\in\modelsOf{\alpha} \).  
			
			If \( \modelsOf{\alpha} \) is a singleton set, then we obtain \( \modelsOf{\Psi\clRevision\alpha}=\emptyset \). From consistency of \( \alpha \) and \eqref{pstl:GCLR3} we obtain that \( \beliefsOf{\Psi} \) is inconsistent, showing that \( \modelsOf{\Psi}=\modelsOf{\Psi\clRevision\alpha} \) holds.

			We consider the remaining case of \( \modelsOf{\alpha}=\{\omega_1,\ldots,\omega_n\} \) for \( n>1 \).
			Towards a contradiction, suppose \( \modelsOf{\Psi}\neq\modelsOf{\Psi\clRevision\alpha} \). Thus, by \eqref{pstl:GCLR1} we have \( \modelsOf{\Psi\clRevision\alpha}\subseteq \modelsOf{\alpha} \).
			We consider two subcases:
			\begin{itemize}\setlength{\itemsep}{0pt}
				\item[]\hspace{-1.75ex}\emph{\( \modelsOf{\Psi\clRevision\alpha}=\emptyset \).} As before, from consistency of \( \alpha \) and \eqref{pstl:GCLR3} we obtain that \( \beliefsOf{\Psi} \) is inconsistent, showing that \( \modelsOf{\Psi}=\modelsOf{\Psi\clRevision\alpha} \) holds.
				\item[]\hspace{-1.75ex}\emph{\( \modelsOf{\Psi\clRevision\alpha}\neq\emptyset \).} Let \( \omega\in\modelsOf{\Psi\clRevision\alpha} \) and let  \( \gamma_{\omega} \) be a formula such that \( \alpha\equiv\gamma_{\omega}\lor\varphi_\omega \) and \( \omega\notin\modelsOf{\gamma_{\omega}} \).
				Note that by \eqref{pstl:GCLR1} we have \( \modelsOf{\Psi\clRevision\gamma_{\omega}}=\modelsOf{\Psi} \) or \( \modelsOf{\Psi\clRevision\gamma_{\omega}}\subseteq\modelsOf{\gamma_{\omega}} \).
				This together with \( \omega\notin\modelsOf{\Psi} \) implies \( \omega\notin \modelsOf{\Psi\clRevision\gamma_{\omega}} \). 
				By \eqref{pstl:GCLR5} we obtain \( \modelsOf{\Psi\clRevision\alpha}=\modelsOf{\Psi\clRevision(\gamma_{\omega}\lor\varphi_{\omega})} \).
				Because \( \clRevision \) satisfies \eqref{pstl:GCLR7} and \( \omega\notin \modelsOf{\Psi\clRevision\gamma_{\omega}} \) holds, we have \( \modelsOf{\Psi\clRevision\alpha}=\modelsOf{\Psi\clRevision\varphi_{\omega}}  \) or \( \modelsOf{\Psi\clRevision\alpha}=\modelsOf{\Psi\clRevision\gamma_{\omega}}\cup\modelsOf{\Psi\clRevision\varphi_{\omega}}  \).
				From this observation, we obtain \( \modelsOf{\Psi\clRevision\alpha}\cap\{\omega\}=\modelsOf{\Psi\clRevision\varphi_{\omega}}=\{\omega\} \), a contradiction to \(  \modelsOf{\Psi\clRevision\varphi_{\omega}}\neq\{\omega\} \).
			\end{itemize}

		\end{itemize}

		\emph{The \enquote{\( \Leftarrow \)}-direction.} 
		In the following, let \( \clRevision \) be a belief change operator for \( \epistemicSpace \) and let \( \Psi \mapsto {(\preceq_\Psi,C_\Psi, b_{\Psi})}  \) be a faithful credibility-limited assignment  that is \ref{eq:cl_revision} with \( \clRevision \).
		We show satisfaction of \eqref{pstl:GCLR1}, \eqref{pstl:GCLR2}, \eqref{pstl:GCLR4}, \eqref{pstl:GCLR3} and \eqref{pstl:GCLR5}--\eqref{pstl:GCLR7}.
		\begin{itemize}\setlength{\itemsep}{0pt}
			\item[]\hspace{-1.75ex}\eqref{pstl:GCLR1} Due to the credibility-limited-compatibility of \( \Psi \mapsto (\preceq_\Psi,C_\Psi, b_{\Psi})  \) with \( \clRevision \) there are three cases to consider.
			If \( \modelsOf{\alpha} \cap C_\Psi \neq \emptyset \), we obtain \( \modelsOf{\Psi\clRevision\alpha}\subseteq\modelsOf{\alpha} \). Consequently, we have that \( \alpha\in\beliefsOf{\Psi\clRevision\alpha} \). 
			If \( {\modelsOf{\alpha}=\emptyset} \) and \( {b_\Psi=\bot} \) holds, we obtain \( \modelsOf{\Psi\clRevision\alpha} = \emptyset \). Consequently, we have \( \beliefsOf{\Psi\clRevision\alpha}=\Cn(\bot) \), and thus \( \alpha\in\beliefsOf{\Psi\clRevision\alpha} \) holds. 
			If none of the cases above applies, we obtain \( \modelsOf{\Psi\clRevision\alpha} = \modelsOf{\Psi} \), which is equivalent to \( \beliefsOf{\Psi\clRevision\alpha}=\beliefsOf{\Psi} \).

			\item[]\hspace{-1.75ex}\eqref{pstl:GCLR2} Suppose that \( \beliefsOf{\Psi}+\alpha \) is consistent, i.e., \( \modelsOf{\Psi}\cap\modelsOf{\alpha} \) is non-empty.
			Due to the faithfulness of \( \Psi \mapsto (\preceq_\Psi,C_\Psi, b_{\Psi}) \), we obtain that \( {\modelsOf{\Psi}\cap\modelsOf{\alpha}} = {\minOf{\modelsOf{\alpha}}{\preceq_{\Psi}}} \) holds.
			From the credibility-limited-compatibility of \( \Psi \mapsto (\preceq_\Psi,C_\Psi, b_{\Psi})  \) with \( \clRevision \) we obtain \( \modelsOf{\Psi\clRevision\alpha}=\modelsOf{\Psi}\cap\modelsOf{\alpha} \).
			This is equivalent to \( \beliefsOf{\Psi\clRevision\alpha}=\beliefsOf{\Psi}+\alpha \).
			
			\item[]\hspace{-1.75ex}\eqref{pstl:GCLR5} Let \( \alpha\equiv\beta \), i.e., \( \modelsOf{\alpha}=\modelsOf{\beta} \). From credibility-limited-compatibility of \( \Psi \mapsto (\preceq_\Psi,C_\Psi, b_{\Psi})  \) with \( \clRevision \) we immediately obtain \( \modelsOf{\Psi\clRevision\alpha}=\modelsOf{\Psi\clRevision\beta} \).
			
			\item[]\hspace{-1.75ex}\eqref{pstl:GCLR4} Suppose that \( \beliefsOf{\Psi\clRevision\alpha} \) is consistent and \( \alpha\models\beta \) holds, i.e., we have \( \modelsOf{\Psi\clRevision\alpha}\neq\emptyset \) and \( \modelsOf{\alpha}\subseteq\modelsOf{\beta} \). 
			We show that \( \beliefsOf{\Psi\clRevision\beta} \) is consistent.
			If \( \alpha\equiv\beta \), we obtain the claim directly from \eqref{pstl:GCLR5}. In the following we assume \( \modelsOf{\alpha}\subsetneq \modelsOf{\beta} \).
			Consequently, we have that \( \modelsOf{\beta}\neq\emptyset \).
			From credibility-limited-compatibility of \( \Psi \mapsto (\preceq_\Psi,C_\Psi, b_{\Psi})  \) with \( \clRevision \) and consistency of \( \modelsOf{\Psi\clRevision\alpha} \) we obtain that either \( \modelsOf{\Psi\clRevision\alpha}=\min(\modelsOf{\alpha},\preceq_{\Psi}) \) or \( \modelsOf{\Psi\clRevision\alpha}=\modelsOf{\Psi} \) holds.
			We consider three cases:
			\begin{itemize}\setlength{\itemsep}{0pt}
				\item[]\hspace{-1.75ex}\emph{\( \alpha \) is inconsistent.} In this case, we have \( \modelsOf{\Psi\clRevision\alpha}=\modelsOf{\Psi} \). As a direct consequence, we obtain from credibility-limited-compatibility that \( \modelsOf{\Psi}\neq\emptyset \) holds. 
				Recalling that \( \beta \) is consistent, consultation of credibility-limited-compatibility reveals that \( \beliefsOf{\Psi\clRevision\beta} \) is consistent in all cases.
				\item[]\hspace{-1.75ex}\emph{\( \alpha \) is consistent and \( \modelsOf{\Psi\clRevision\alpha}=\min(\modelsOf{\alpha},\preceq_{\Psi}) \).} 
				From credibility-limited-compatibility of \( \Psi \mapsto (\preceq_\Psi,C_\Psi, b_{\Psi})  \) with \( \clRevision \) we obtain \( {\modelsOf{\alpha}\cap C_\Psi \neq\emptyset} \). 
				Consequently, we also have \( \modelsOf{\beta}\cap C_\Psi \neq\emptyset \). Consulting credibility-limited-compatibility yields that \( \beliefsOf{\Psi\clRevision\beta} \) is consistent. 
				\item[]\hspace{-1.75ex}\emph{\( \alpha \) is consistent and \( \modelsOf{\Psi\clRevision\alpha}=\modelsOf{\Psi} \).}
				From consistency of \( \alpha \) we obtain consistency of \( \beta \). Thus, by credibility-limited-compatibility, we obtain that \( \beliefsOf{\Psi\clRevision\beta} \) is consistent.
			\end{itemize}
			
			\item[]\hspace{-1.75ex}\eqref{pstl:GCLR3} Suppose that \( \beliefsOf{\Psi\clRevision\alpha} \) is inconsistent.
			We show that \( \beliefsOf{\Psi} \)  is inconsistent or \( \alpha \) is inconsistent, by obtaining a contradiction when assuming the contrary, i.e., \( \beliefsOf{\Psi} \)  is consistent and \( \alpha \) is consistent. From credibility-limited-compatibility of \( \Psi \mapsto (\preceq_\Psi,C_\Psi, b_{\Psi})  \) with \( \clRevision \) and consistency of \( \alpha \), we obtain \( \modelsOf{\Psi\clRevision\alpha}=\modelsOf{\Psi} \), a contradiction between the inconsistency of \( \beliefsOf{\Psi\clRevision\alpha} \) and the consistency of \( \beliefsOf{\Psi} \).

			\item[]\hspace{-1.75ex}\eqref{pstl:GCLR6} Suppose \( \alpha\in\beliefsOf{\Psi\clRevision\alpha} \) and \( \alpha\models\beta \), i.e., \( \modelsOf{\Psi\clRevision\alpha}\subseteq \modelsOf{\alpha} \) and \( \modelsOf{\alpha}\subseteq\modelsOf{\beta} \). We show \( \beta\in\beliefsOf{\Psi\clRevision\beta} \).
			If \( \modelsOf{\beta}\cap C_\Psi \) is non-empty, then we obtain \( \beta\in\beliefsOf{\Psi\clRevision\beta} \). If \( \modelsOf{\beta}=\emptyset \), then \( \modelsOf{\alpha}=\emptyset \). We obtain from \( \modelsOf{\Psi\clRevision\alpha}\subseteq \modelsOf{\alpha} \) that \( \beta\in\beliefsOf{\Psi\clRevision\beta} \) holds.
			In the remaining case of \( \modelsOf{\beta}\cap C_\Psi = \emptyset \) and \( \beta \) is consistent, we obtain \( \modelsOf{\alpha}\cap C_\Psi = \emptyset \) from \( \modelsOf{\alpha}\subseteq\modelsOf{\beta} \). 
			From \( \modelsOf{\Psi\clRevision\alpha}\subseteq \modelsOf{\alpha} \) and the credibility-limited-compatibility of \( \Psi \mapsto (\preceq_\Psi,C_\Psi, b_{\Psi})  \) with \( \clRevision \) we obtain two cases:
			\begin{itemize}\setlength{\itemsep}{0pt}
				\item[]\hspace{-1.75ex}\emph{\( \modelsOf{\Psi\clRevision\alpha}=\modelsOf{\Psi} \).} Using the consistency of \( \beta \) and the credibility-limited-compatibility of \( \Psi \mapsto (\preceq_\Psi,C_\Psi, b_{\Psi})  \) with \( \clRevision \) again, we obtain \( {\modelsOf{\Psi\clRevision\beta}}=\modelsOf{\Psi} \). From \( \modelsOf{\Psi}=\modelsOf{\Psi\clRevision\alpha}\subseteq \modelsOf{\alpha} \) and \( \modelsOf{\alpha}\subseteq\modelsOf{\beta} \) we obtain \( {\modelsOf{\Psi\clRevision\alpha}}=\modelsOf{\Psi}=\modelsOf{\Psi\clRevision\beta}\subseteq\modelsOf{\beta} \).
				\item[]\hspace{-1.75ex}\emph{\( \modelsOf{\alpha}=\emptyset \) and \( b_\Psi =\bot \).}
				Because \( \beta \) is consistent and \eqref{pstl:CLAbot} holds, we obtain \( C_\Psi=\Omega \), which is a contradiction to \( \modelsOf{\beta}\cap C_\Psi = \emptyset \).
			\end{itemize}        

			\item[]\hspace{-1.75ex}\eqref{pstl:GCLR7}
			In the following, suppose that \( \alpha,\beta,\gamma \) are formulas with \( \gamma=\alpha\lor\beta \).
			We show that \( \modelsOf{\Psi\clRevision\gamma}=\modelsOf{\Psi\clRevision\alpha} \) or \( \modelsOf{\Psi\clRevision\gamma}=\modelsOf{\Psi\clRevision\beta} \) or \( \modelsOf{\Psi\clRevision\gamma}=\modelsOf{\Psi\clRevision\alpha}\cup\modelsOf{\Psi\clRevision\beta} \) holds.
			Note that by credibility-limited-compatibility of \( \Psi \mapsto (\preceq_\Psi,C_\Psi, b_{\Psi})  \) with \( \clRevision \) there are several cases for each of \( \alpha \) and \( \beta \).
			In the following, we consider (potentially overlapping) cases, all other not explicitly mentioned cases will follow by \eqref{pstl:GCLR5} and symmetry:
			\begin{itemize}\setlength{\itemsep}{0pt}
				\item[]\hspace{-1.75ex}\emph{The case of \( \modelsOf{\alpha}\cap C_\Psi\neq\emptyset \) and \( \modelsOf{\beta}\cap C_\Psi\neq\emptyset \).} For this case observe
				that \( \modelsOf{\gamma}\cap C_\Psi = (\modelsOf{\alpha}\cap C_\Psi)\cup(\modelsOf{\beta}\cap C_\Psi) \) holds.
				We obtain satisfaction of \eqref{pstl:GCLR7} by Lemma \ref{lem:sem_trichotonomy}.

				\item[]\hspace{-1.75ex}\emph{The case of \( \modelsOf{\alpha}\cap C_\Psi\neq\emptyset \) and \( \modelsOf{\beta}\cap C_\Psi=\emptyset \).} 
				In this case, we obtain \( \modelsOf{\gamma}\cap C_\Psi=\modelsOf{\alpha}\cap C_\Psi \) from  \( \modelsOf{\alpha}\cap C_\Psi\neq\emptyset \). Considering credibility-limited-compatibility yields that \( \modelsOf{\Psi\clRevision\gamma}=\modelsOf{\Psi\clRevision\alpha} \) holds.
				\item[]\hspace{-1.75ex}\emph{The case of \( \modelsOf{\alpha}=\modelsOf{\beta} \).} 
				We obtain that \( \modelsOf{\gamma}=\modelsOf{\alpha}=\modelsOf{\beta} \). From \eqref{pstl:GCLR5} we obtain \( \modelsOf{\Psi\clRevision\gamma}=\modelsOf{\Psi\clRevision\alpha}=\modelsOf{\Psi\clRevision\beta} \).
				\item[]\hspace{-1.75ex}\emph{The case of \( \modelsOf{\alpha}\cap C_\Psi=\emptyset \), and \( \modelsOf{\beta}\cap C_\Psi=\emptyset \) and  \( \modelsOf{\alpha}\neq\modelsOf{\beta} \).} 
				We obtain \( \modelsOf{\gamma}=\modelsOf{\alpha}\cup\modelsOf{\beta}\neq\emptyset \) from \( \modelsOf{\alpha}\neq\modelsOf{\beta} \).
				Because \( \modelsOf{\alpha}\cap C_\Psi=\emptyset \) and \( \modelsOf{\beta}\cap C_\Psi=\emptyset \) hold, we have \( \modelsOf{\gamma}\cap C_\Psi=\emptyset \). We obtain \( \modelsOf{\Psi\clRevision\gamma}=\modelsOf{\Psi} \). Moreover, we have that \( \modelsOf{\alpha}\neq\emptyset \) or \( \modelsOf{\beta}\neq\emptyset \) holds. This implies that  \( \modelsOf{\Psi\clRevision\alpha}=\modelsOf{\Psi} \) or \( \modelsOf{\Psi\clRevision\beta}=\modelsOf{\Psi} \) holds.
				We obtain that at last one of \( \modelsOf{\Psi\clRevision\gamma}=\modelsOf{\Psi\clRevision\alpha} \) or  \( \modelsOf{\Psi\clRevision\gamma}=\modelsOf{\Psi\clRevision\beta} \) holds. \qedhere
			\end{itemize}
		\end{itemize}
\end{proof}  
\end{document}